\documentclass{article} % For LaTeX2e
\usepackage{iclr2023_conference,times}

\usepackage{amsmath,amsfonts,bm}

\def\eqref#1{equation~\ref{#1}}
\def\1{\bm{1}}

\DeclareMathAlphabet{\mathsfit}{\encodingdefault}{\sfdefault}{m}{sl}
\SetMathAlphabet{\mathsfit}{bold}{\encodingdefault}{\sfdefault}{bx}{n}

\usepackage[utf8]{inputenc} % allow utf-8 input
\usepackage[T1]{fontenc}    % use 8-bit T1 fonts
\usepackage{hyperref}       % hyperlinks
\usepackage{url}            % simple URL typesetting
\usepackage{booktabs}       % professional-quality tables
\usepackage{amsfonts}       % blackboard math symbols
\usepackage{nicefrac}       % compact symbols for 1/2, etc.
\usepackage{microtype}      % microtypography
\usepackage{xcolor}         % colors
\usepackage{amsmath}
\usepackage{adjustbox}
\usepackage[noend]{algpseudocode}
\usepackage{algorithmicx,algorithm}
\usepackage{graphicx}
\usepackage{wrapfig}
\usepackage{subfigure}
\usepackage{multirow}
\usepackage{float}
\usepackage{footnote}
\usepackage{verbatim}
\usepackage{amsthm}
\usepackage{threeparttable}
\usepackage{bm}
\usepackage{amssymb}
\usepackage{soul}
\usepackage{xspace} 
\usepackage{enumitem}
\usepackage{chngcntr}
\usepackage{apptools}
\usepackage{ulem}

\newtheorem{definition}{Definition}
\newtheorem{assumption}{Assumption}

\newtheorem{theorem}{Theorem}

\newtheorem{lemma}{Lemma}

\newtheorem{remark}{Remark}
\AtAppendix{\counterwithin{lemma}{section}}

\newcommand{\name}[0]{NGC\xspace}

\title{How Powerful is Implicit Denoising in Graph Neural Networks}

\iclrfinalcopy
\author{Songtao Liu\textsuperscript{1}, Rex Ying\textsuperscript{2}, Hanze Dong\textsuperscript{3}, Lu Lin\textsuperscript{1}, Jinghui Chen\textsuperscript{1}, Dinghao Wu\textsuperscript{1} \\
The Pennsylvania State University$^1$, Yale University$^2$,\\The Hong Kong University of Science and Technology$^3$\\
\texttt{\{skl5761,lulin,jzc5917,dinghao\}@psu.edu}\\\texttt{rex.ying@yale.edu}, \texttt{hdongaj@ust.hk}\\
}

\begin{document}

\maketitle

\begin{abstract}
Graph Neural Networks (GNNs), which aggregate features from neighbors, are widely used for graph-structured data processing due to their powerful representation learning capabilities. It is generally believed that GNNs can implicitly remove the non-predictive noises. However, the analysis of implicit denoising effect in graph neural networks remains open. In this work, we conduct a comprehensive theoretical study and analyze when and why the implicit denoising happens in GNNs. Specifically, we study the convergence properties of noise matrix. Our theoretical analysis suggests that the implicit denoising largely depends on the connectivity, the graph size, and GNN architectures. Moreover, we formally define and propose the adversarial graph signal denoising (AGSD) problem by extending graph signal denoising problem. By solving such a problem, we derive a robust graph convolution, where the smoothness of the node representations and the implicit denoising effect can be enhanced. Extensive empirical evaluations verify our theoretical analyses and the effectiveness of our proposed model. 

\end{abstract}

\section{Introduction}
Graph Neural Networks (GNNs)~\citep{kipf2017semi,velivckovic2018graph,hamilton2017inductive} have been widely used in graph learning and achieved remarkable performance on graph-based tasks, such as traffic prediction~\citep{guo2019attention}, drug discovery~\citep{dai2019retrosynthesis}, and recommendation system~\citep{ying2018graph}. A general principle behind Graph Neural Networks (GNNs)~\citep{kipf2017semi,velivckovic2018graph,hamilton2017inductive} is to perform a message passing operation that aggregates node features over neighborhoods, such that the smoothness of learned node representations on the graph is enhanced. 

By promoting graph smoothness, the message passing and aggregation mechanism naturally leads to GNN models whose predictions are not only dependent on the feature of one specific node, but also the features from a set of neighboring nodes. Therefore, this mechanism can, to a certain extent, protect GNN models from noises: real-world graphs are usually noisy, e.g., Gaussian white noise exists on node features~\citep{zhou2021graph}, however, the influence of feature noises on the model's output could be counteracted by the feature aggregation operation in GNNs. We term this effect as \textit{implicit denoising}.

While many works have been conducted in the empirical exploration of GNNs, relatively fewer advances have been achieved in theoretically studying this denoising effect. Early GNN models, such as the vanilla GCN~\citep{kipf2017semi}, GAT~\citep{velivckovic2018graph} and GraphSAGE~\citep{hamilton2017inductive}, propose different designs of aggregation functions, but the denoising effect is not discussed in these works. Some recent attempts~\citep{ma2021unified} are made to mathematically establish the connection between a variety of GNNs and the graph signal denoising problem (GSD)~\citep{chen2014signal}:
\begin{equation}
\label{graph signal denoising}
    q(\mathbf{F}) = \min _{\mathbf{F}} \|\mathbf{F}-\mathbf{X}\|_{F}^{2} + \lambda \operatorname{tr}\left(\mathbf{F}^{\top} \widetilde{\mathbf{L}} \mathbf{F}\right),
\end{equation}
where $\mathbf{X}=\mathbf{X}^{*}+\bm{\eta}$ is the observed noisy feature matrix, $\bm{\eta} \in \mathbb{R}^{n \times d}$ is the noise matrix, $\mathbf{X}^{*}$ is the clean feature matrix, and $\widetilde{\mathbf{L}}$ is the graph Laplacian. The second term encourages the smoothness of the filtered feature matrix $\mathbf{F}$ over the graph., i.e., nearby vertices should have similar vertex features. By regarding the feature aggregation process in GNNs as solving a GSD problem, more advanced GNNs are proposed, such as GLP~\citep{li2019label}, S$^2$GC~\citep{zhu2021simple}, and IRLS~\citep{yang2021graph}.
Despite these prior attempts, little efforts have been made to rigorously study the denoising effect of message passing and aggregation operation. This urges us to think about a fundamental but not clearly answered question:
\begin{center}
    \textit{Why and when implicit denoising happens in GNNs?}
\end{center}

In this work, we focus on the non-predictive stochasticity of noise in GNNs' aggregated features and analyze its properties. We prove that with the increase in graph size and graph connectivity factor, the stochasticity tends to diminish, which is called the ``denoising effect'' in our work. We will address this question using the tools from concentration inequalities and matrix theories, which are concerned with the study of the convergence of noise matrix. It offers a new framework to study the properties of graphs and GNNs in terms of the denoising effect. In order to facilitate our theoretical analysis, we derive Neumann Graph Convolution (NGC) from GSD. Specifically, to study the convergence rate, we introduce an insightful measurement on the convolution operator, termed \textit{high-order graph connectivity factor}, which reveals how uniform the nodes are distributed in the neighborhood and reflects the strength of information diluted on a single neighboring node during the feature aggregation step. Intuitively, as the General Hoeffding Inequality~\citep{hoeffding1994probability} (Lemma.~\ref{hoeffding}) suggests, a larger high-order graph connectivity factor, i.e., nodes are more uniformly distributed in the neighborhood, accelerates the convergence of the noise matrix and a larger graph size leads to faster convergence. Besides, GNN architectures also affect the convergence rate. Deeper GNNs can have a faster convergence rate.

To further strengthen the denoising effect, inspired by the adversarial training method~\citep{madry2018towards}, we propose the adversarial graph signal denoising problem (AGSD). By solving such a problem, we derive a robust graph convolution model based on the correlation of node feature and graph structure to increase the high-order graph connectivity factor, which helps us improve the denoising performance. Extensive experimental results on standard graph learning tasks verify our theoretical analyses and the effectiveness of our derived robust graph convolution model.

\paragraph{Notations.} Let $\mathcal{G}=(\mathcal{V}, \mathcal{E})$ represent a undirected graph, where $\mathcal{V}$ is the set of vertices $\left\{v_{1}, \cdots, v_{n}\right\}$ with $|\mathcal{V}|=n$ and $\mathcal{E}$ is the set of edges. The adjacency matrix is defined as $\mathbf{A} \in \{0,1\}^{n \times n}$, and $\mathbf{A}_{i,j}=1$ if and only if $\left(v_{i}, v_{j}\right) \in \mathcal{E}$. Let $\mathcal{N}_{i}=\{v_j | \mathbf{A}_{i,j} = 1\}$ denote the neighborhood of node $v_i$ and $\mathbf{D}$ denote the diagonal degree matrix, where $\mathbf{D}_{i,i}=\sum_{j=1}^{n} \mathbf{A}_{i,j}$. The feature matrix is denoted as $\mathbf{X} \in \mathbb{R}^{n \times d}$ where each node $v_i$ is associated with a $d$-dimensional feature vector $\mathbf{X}_i$. $\mathbf{Y} \in \{0,1\}^{n \times c}$ denotes the matrix, where $\mathbf{Y}_i \in \{0, 1\}^{c}$ is a one-hot vector and $\sum_{j=1}^{c} \mathbf{Y}_{i,j}=1$ for any $v_i \in V$.

\section{A Simple Unifying Framework: Neumann Graph Convolution}
\label{sec:NGC}
\paragraph{A General Framework.} In this section, we discuss a simple yet general framework for solving graph signal denoising problem, namely Neumann Graph Convolution (\name). Note that \name is not a new GNN architecture. There also exist similar GNN architectures, such as GLP~\citep{li2019label}, S$^2$GC~\citep{zhu2021simple}, and GaussianMRF~\citep{jia2022unifying}. We focus on the theoretical analysis of the denoising effect in GNNs in this work. \name can facilitate our theoretical analysis. By taking the derivative $\nabla q\left(\mathbf{F}\right)=2 \widetilde{\mathbf{L}} \mathbf{F}+2(\mathbf{F}-\mathbf{X})$ to zero, we obtain the solution of GSD optimization problem as follows:
\begin{equation}\label{eq:inverse}
    \mathbf{F} = (\mathbf{I}+\lambda\widetilde{\mathbf{L}})^{-1} \mathbf{X}.
\end{equation}
To avoid the expensive computation of the inverse matrix, we can use Neumann series~\citep{stewart1998matrix} expansion to approximate Eq.~(\ref{eq:inverse}) up to up to $S$-th order:
\begin{equation}
\label{Neum Series}
    \left(\mathbf{I}+\lambda \widetilde{\mathbf{L}}\right)^{-1} =\frac{1}{\lambda+1}\left(\mathbf{I}-\frac{\lambda}{\lambda+1}\widetilde{\bm{\mathcal{A}}}\right)^{-1}\approx\frac{1}{\lambda+1}\sum_{s=0}^{S} \left(\frac{\lambda}{\lambda+1}\widetilde{\bm{\mathcal{A}}}\right)^{s},
\end{equation}
where $\widetilde{\bm{\mathcal{A}}}$ can take the form of $\widetilde{\bm{\mathcal{A}}}=\widetilde{\mathbf{D}}^{-\frac{1}{2}}\widetilde{\mathbf{A}} \widetilde{\mathbf{D}}^{-\frac{1}{2}}$ or $\widetilde{\bm{\mathcal{A}}}=\widetilde{\mathbf{D}}^{-1}\widetilde{\mathbf{A}}$, and the proof can be found in Appendix~\ref{appendix:neumann series}. Based on the Neumann series expansion of the solution of GSD, we introduce a general graph convolution model -- Neumann Graph Convolution defined as the following expansion: 

\begin{equation}
\label{NGC filter}
    \mathbf{H}=\widetilde{\bm{\mathcal{A}}}_{S}\mathbf{X}\mathbf{W} =\frac{1}{\lambda+1}\sum_{s=0}^{S} \left(\frac{\lambda}{\lambda+1}\widetilde{\bm{\mathcal{A}}}\right)^{s}\mathbf{X}\mathbf{W},
\end{equation}
where $\widetilde{\bm{\mathcal{A}}}_{S}=\frac{1}{\lambda+1}\sum_{s=0}^{S} \left(\frac{\lambda}{\lambda+1}\widetilde{\bm{\mathcal{A}}}\right)^{s}$ and $\mathbf{W}$ is the weight matrix. Our spectral convolution $\widetilde{\bm{\mathcal{A}}}_{S}\mathbf{X}$ on graphs is a multi-scale graph convolution~\citep{abu2020n,liao2019lanczosnet}, which covers the single-scale graph convolution models such as SGC~\citep{wu2019simplifying} since the graph convolution of SGC is $\widetilde{\bm{\mathcal{A}}}^2\mathbf{X}$ and $\widetilde{\bm{\mathcal{A}}}^2$ is the third term of $\widetilde{\bm{\mathcal{A}}}_{S}$. Besides, if we remove the non-linear functions in GCN~\citep{kipf2017semi}, it also can be covered by our model. Therefore, we can draw the conclusion that our proposed NGC is a general framework.

\paragraph{High-order Graph Connectivity Factor.} Based on NGC, we obtain the filtered graph signal via $\mathbf{F}=\widetilde{\bm{\mathcal{A}}}_{S}\mathbf{X}$. Intuitively, $\widetilde{\bm{\mathcal{A}}}_{S}$ captures not only the connectivity of the graph structure (represented by $\widetilde{\bm{\mathcal{A}}}$), but also the higher order connectivity (represented by $\widetilde{\bm{\mathcal{A}}}^2, \widetilde{\bm{\mathcal{A}}}^3, \ldots, \widetilde{\bm{\mathcal{A}}}^S$). As will be discussed in Sec.~\ref{sec:main theory}, larger high order graph connectivity can accelerate the convergence of the noise feature matrix. To formally quantify the high order graph connectivity, we give the following definition:
\begin{definition}[High-order Graph Connectivity Factor]
We define the high-order graph connectivity factor $\tau$ as
\begin{equation}
\label{evenness factor}
    \tau = \max_i \tau_i, \ \ \text{ where } \tau_i = {n\sum_{j=1}^{n}\left[\widetilde{\bm{\mathcal{A}}}_{S}\right]_{ij}^2}\Bigg/{\left(1-\left(\frac{\lambda}{\lambda+1}\right)^{S+1}\right)^2}. 
\end{equation}
\end{definition}
\begin{remark}
Here we give some intuitions about why Eq.~(\ref{evenness factor}) represents high-order graph connectivity. Note that each element in $\widetilde{\bm{\mathcal{A}}}_{S}$ is non-negative and each row sum satisfies\footnote{Note that this result is obtained by using $\widetilde{\bm{\mathcal{A}}}=\widetilde{\mathbf{D}}^{-1}\widetilde{\mathbf{A}}$ for the ease of theoretical analysis while in experiments we adopt more commonly used $\widetilde{\bm{\mathcal{A}}}=\widetilde{\mathbf{D}}^{-\frac{1}{2}}\widetilde{\mathbf{A}} \widetilde{\mathbf{D}}^{-\frac{1}{2}}$. The proof can be found in Appendix~\ref{appendix:row sum}.} 
\begin{equation}
\label{row sum}
    \sum_{j=1}^{n}\left[\widetilde{\bm{\mathcal{A}}}_{S}\right]_{ij} 
    =1-\left(\frac{\lambda}{\lambda+1}\right)^{S+1}.
\end{equation}
Based on Eq.~(\ref{row sum}), the sum of squares of elements in each row satisfy:
\begin{equation}
\label{range}
   {\left(1-\left(\frac{\lambda}{\lambda+1}\right)^{S+1}\right)^2}\bigg/{n} \leq \sum_{j=1}^{n}\left[\widetilde{\bm{\mathcal{A}}}_{S}\right]_{ij}^2 \leq \left(1-\left(\frac{\lambda}{\lambda+1}\right)^{S+1}\right)^2.
\end{equation}
When the high-order graph has a high connectivity, i.e., the elements in row $i$ of $\widetilde{\bm{\mathcal{A}}}_{S}$ are more uniformly distributed, Eq.~(\ref{range}) reaches its lower bound. Meanwhile, if the graph is not connected and there is only one element whose value is larger than $0$ in row $i$, Eq.~(\ref{range}) reaches its upper bound. Therefore, the value of $\tau \in [1, n]$ is determined as follows: when the high-order graph connectivity is high, $\tau \to 1$ and when the graph is less connected, $\tau \to n$.
\end{remark}

\section{Main Theory}
\label{sec:main theory}
In this section, we analyze the denoising effect of \name. Before we present our main theory, we first present our aggregation on noisy feature matrix and formulate four assumptions, which are necessary to construct our theory.

For the convenience of theoretical analysis, we adopt MSE loss\footnote{We consider MSE loss since it gives easier form of gradient and it can be extended to other losses satisfying certain conditions.} for our main theory.  
Consider $\widetilde{\bm{\mathcal{A}}}_{S}$ as our aggregation scheme, the NGC training based on Eq.~(\ref{NGC filter}) can be formulated as
\begin{equation}
\label{loss function}
\min_{\mathbf{W}} f(\mathbf{W}) =  \left\|\widetilde{\bm{\mathcal{A}}}_{S}\mathbf{X}\mathbf{W} - \mathbf{Y}\right\|_{F}^{2} = \left\|\widetilde{\bm{\mathcal{A}}}_{S}(\mathbf{X}^{*} + \bm{\eta} )\mathbf{W} - \mathbf{Y}\right\|_{F}^{2},
\end{equation}
where $\mathbf{X}^{*}$ is the clean feature matrix, $\bm{\eta}$ denotes the noise added on $\mathbf{X}^{*}$, and $\mathbf{X} = \mathbf{X}^{*} + \bm{\eta}$ is the observed data matrix. Intuitively, if $\widetilde{\bm{\mathcal{A}}}_{S}\bm{\eta}$ is small enough, the added noise will not change the optimization direction on which the parameter is updated under the clean feature matrix $\mathbf{X}^{*}$. Before we present our main theory, we give four assumptions about noise $\bm{\eta}$, $\widetilde{\bm{\mathcal{A}}}_{S}$, and parameters $\mathbf{W}$.

\begin{assumption}
\label{sub-Gaussian assumption}
Each entry of the noise matrix $\bm{\eta}$, i.e., $[\bm{\eta}]_{ij}$ is i.i.d sub-Gaussian random variable with variance $\sigma$ and mean $\mu=0$, i.e., 
\begin{equation}
    \mathbb{E}\left[e^{\lambda([\bm{\eta}]_{ij}-\mu)}\right] \leq e^{\sigma^{2} \lambda^{2} / 2} \quad \text { for all } \lambda \in \mathbb{R}.
\end{equation}
\end{assumption}
Note that it is common to assume that the noise follows Gaussian distribution~\citep{zhou2021graph,chen2021graph,zhang2022graphless}, which is also covered by our sub-Gaussian assumption.

\begin{assumption}
\label{evenness factor assumption}
The high-order graph connectivity factor $\tau$ 
is $\mathcal{O}\left(n\right)$, i.e.,
$
    \lim_{n\rightarrow \infty}\frac{\tau}{n}=0.
$
\end{assumption}
As we have discussed in Sec.~\ref{sec:NGC}, $\tau$ depends on the graph structure. In a well-connected graph, $\tau$ is usually relatively small compared with $n$. Only if all nodes of a graph are isolated, $\tau$ reaches its upper bound $n$.
\begin{assumption}
\label{parameter assumption}
The Frobenius norm of the parameter matrix $\mathbf{W}$ is bounded by a constant. There exists $C>0$ such that
$
    \left\|\mathbf{W}\right\|_{F} \leq C,
$
which is unrelated to $n$.
\end{assumption}
We assume that the Frobenius norm of $\mathbf{W}$ is bounded by a constant. This is reasonable since recent advances in Neural Tangent Kernel~\citep{jacot2018neural} indicate that over-parameterized network weights lie in the neighborhood of the small random initialization, which justifies Assumption~\ref{parameter assumption}. 

\begin{assumption}
\label{loss function assumption}
The loss function in Eq.~(\ref{loss function}) is $L$-smooth, 
\begin{equation}
    \|\nabla f(\mathbf{W}_1)-\nabla f(\mathbf{W}_2)\|_{2} \leq L\|\mathbf{W}_1-\mathbf{W}_2\|_2 \quad \text { for all  } \mathbf{W}_1, \mathbf{W}_2 \in \mathbb{R}^{d\times c}.
\end{equation}
\end{assumption}

The $L$-smoothness of $f$ depends on the largest singular value of $\widetilde{\bm{\mathcal{A}}}_{S}\mathbf{X}$. For conciseness, we start with the smooth case. As the core part of our proof, we first derive the upper bound of the Frobenius norm of $\widetilde{\bm{\mathcal{A}}}_{S}\bm{\eta}$.

\begin{lemma}
\label{upper bound of aggregated noised matirx}
Suppose we choose $t=2\tau\left(1-\left(\frac{\lambda}{\lambda+1}\right)^S\right)^2\left(4\log n+\log{2d}\right)/n$. Then under Assumptions \ref{sub-Gaussian assumption} and \ref{evenness factor assumption}, with a high probability $1-1/d$, we have 
\begin{equation}
\label{concentration bound of aggregated noises}
    \left\|\widetilde{\bm{\mathcal{A}}}_S\bm{\eta}\right\|_{F}^{2}\leq \frac{2\tau\left(1-\left(\frac{\lambda}{\lambda+1}\right)^{S+1}\right)^2\sigma^2\left(4\log n+\log{2d}\right)}{n},
\end{equation}
where the proof can be found in Appendix~\ref{appendix:upper bound with hoeffding}.
\end{lemma}
Lemma~\ref{upper bound of aggregated noised matirx} implies that the norm of the aggregated noise matrix $\widetilde{\bm{\mathcal{A}}}_S\bm{\eta}$ is bounded by three terms: the number of nodes of a graph $n$, the expansion order $S$, the high-order graph connectivity factor $\tau$. Intuitively, as the concentration bounds suggest, if we extract enough samples from the same sub-Gaussian variable, the average of these samples will converge to zero with a high probability. This requires our graph to be large enough and the sum of squares of the elements in the row of $\widetilde{\bm{\mathcal{A}}}_{S}$ to be small enough, which depends on the graph structure. 

Now we start to present our main theorem for graph denoising. In order to demonstrate the effect of graph denoising, we further consider another loss function $g(\cdot)$ with the clean feature matrix:
\begin{equation}
\label{GNN loss}
\begin{split}
    &g(\mathbf{W}) =\left\|\widetilde{\bm{\mathcal{A}}}_{S}\mathbf{X}^{*}\mathbf{W}-\mathbf{Y}\right\|_{F}^{2}.
    \end{split}
\end{equation}

Let $\mathbf{W}_{g}^{*} = \arg\min_{\mathbf{W}} g(\mathbf{W})$ be the minimizer of clean loss $g$, we aim to demonstrate that the learned model (from gradient descent on the noisy data $\mathbf{X}$) has essentially the same performance as $\mathbf{W}_{g}^{*}$ which is the optimal solution for the clean loss $g$.

\begin{theorem}
\label{theorem:main}
Under Assumptions~\ref{sub-Gaussian assumption}, \ref{evenness factor assumption}, \ref{parameter assumption}, \ref{loss function assumption} and Lemma~\ref{upper bound of aggregated noised matirx}, 
let $\mathbf{W}_{f}^{(k)}$ denote the $k$-th step gradient descent solution for $\min_{\mathbf{W}} f(\mathbf{W})$ with step size $\alpha \leq 1 / L$, with probability $1-1/d$ we have
\begin{equation}\label{eq:main}
\begin{split}
    g\left(\mathbf{W}_{f}^{(k)}\right)-g\left(\mathbf{W}_{g}^{*}\right)
    &\leq \mathcal{O}\left(\frac{1}{2k\alpha}\right) + \mathcal{O}\left(\frac{\tau \log n}{n}\right),
\end{split}
\end{equation}
where $\mathbf{W}_{g}^{*} = \arg\min_{\mathbf{W}} g(\mathbf{W})$ is the optimal solution of the clean loss function $g(\mathbf{W})$, $\tau$ is the high-order graph connectivity factor, and $n$ is the number of nodes of a graph. 
\end{theorem}
The proof of Theorem~\ref{theorem:main} can be found in Appendix~\ref{appendix:upper bound with optimization}.
\begin{remark}{\textbf{Denoising Effect.}}
Theorem~\ref{theorem:main} suggests that the $k$-th step gradient descent solution $\mathbf{W}_{f}^{(k)}$ which is trained using the noisy feature matrix $\mathbf{X}$ enjoys a similar performance  as the actual clean loss minimizer $\mathbf{W}_{g}^{*}$ with large enough $k$ and $n$. This implies the denoising effect of our proposed solution in Eq.~(\ref{NGC filter}).
\end{remark}

\begin{remark}\label{remark:graph_on_denoising}
\textbf{Effect of graph structure on denoising.}
Note that the second term in Eq.~(\ref{eq:main}) suggests that the denoising effect is linear with respect to $\tau$, which is directly related to the dataset graph structure. Specifically, as will be shown in Sec.~\ref{graph structure}, refer to Sec.~\ref{graph structure} on how graph structure affects the value of $\tau$, and thus the denoising effect. A large well-connected graph tend to have a better denoising performance since $n$ is large and $\tau$ is close to 1. 
\end{remark}

\begin{figure}[t] 
\centering 
\includegraphics[width=0.75\textwidth]{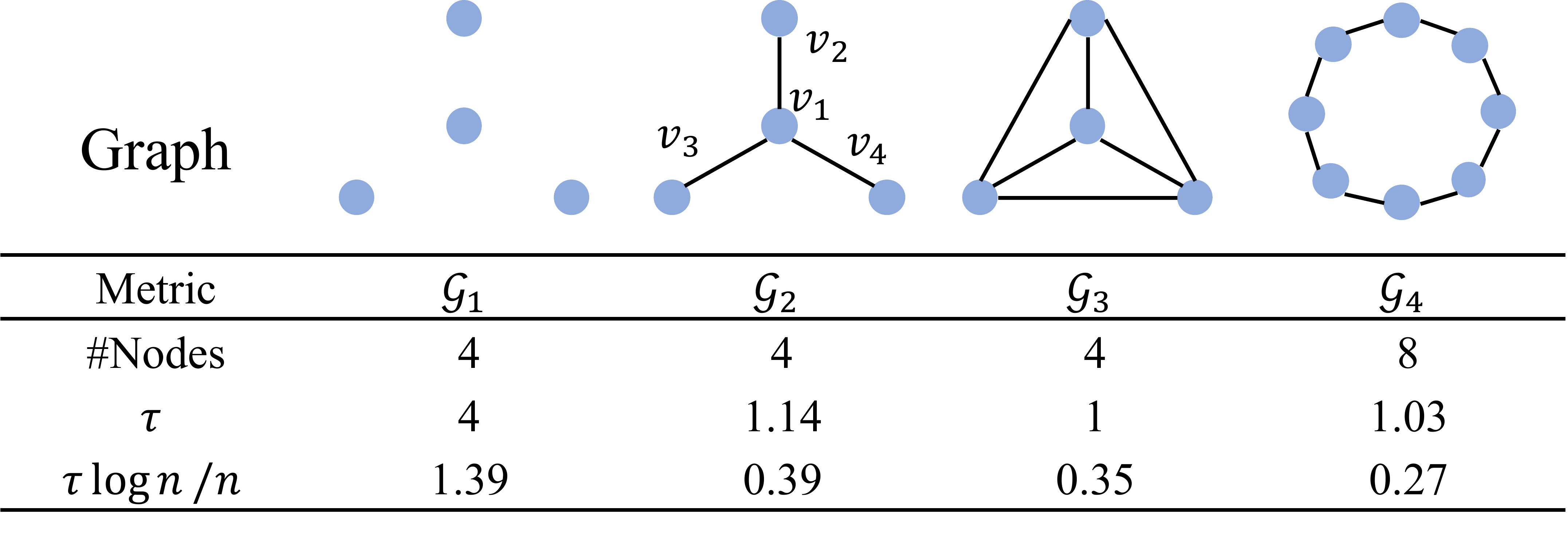} 
\caption{An illustration of the graph structures on the implicit denoising performances. $\mathcal{G}_1$: nodes are isolated; $\mathcal{G}_2$: a star graph with 4 nodes; $\mathcal{G}_3$: a complete graph with 4 nodes; $\mathcal{G}_4$: a ring graph with 8 nodes. For computing $\tau$, $\lambda$ and $S$ are set to be 64.
} 
\label{fig:graph structure} 
\end{figure}

\paragraph{Case Study: the Influence of Graph Structure on Implicit Denoising.}
\label{graph structure}
Remark \ref{remark:graph_on_denoising} suggests that $\tau$ plays an important role on the implicit denoising effect. In this case study, we give four illustration samples in Figure~\ref{fig:graph structure}. $G_1$, $G_2$, and $G_3$ have the same number of nodes. But the nodes on $G_1$ are isolated. $G_2$ has only one connected component and has a center node $v_1$ on the graph. $G_3$ is a complete graph such that there is an edge between any two nodes. In addition, we give a larger illustration graph $G_4$ to understand the influence of graph size. 

From Figure~\ref{fig:graph structure}, we can extract the following insights: 1) There is no denoising effect (the value of $\tau\log n /n $ is quite large) on $G_1$ since the nodes are isolated. And GNNs will degrade to MLP under such graph structure, leading to no aggregation. 2) The complete graph $G_3$ has the best denoising effect among the graphs of the same size if we only consider the influence of graph structure. Since the values of elements in each row are distributed uniformly, leading to the lower bound of $\tau$. 3) Although $G_2$ has only one connected component, there is a center node $v_1$ on the graph. The existence of the center node makes the value of elements in each row imbalanced, which means that $\tau$ tends to have a larger value compared with $G_3$. 4) The decentralized graph like $G_4$ also can get a smaller $\tau$. 5) In terms of graph size, the graph with a larger size has a better denoising effect.

\section{Robust Neumann Graph Convolution} 
\label{sec:RNGC}
In this section, we propose a new graph signal denoising problem - adversarial graph signal denoising (AGSD) problem to improve the denoising performance by deriving a robust graph convolution model.
\subsection{Adversarial Graph Signal Denoising Problem} 
Note that the second term in the GSD problem (Eq.~(\ref{graph signal denoising})) which controls the smoothness of the feature matrix over graphs, is related to both the graph Laplacian and the node features. Therefore, the slight changes in the graph Laplacian matrix could lead to an unstable denoising effect. Inspired by the recent studies in adversarial training~\citep{madry2018towards},
we formulate the adversarial graph signal denoising problem as a min-max optimization problem:
\begin{equation}
\label{adv gsd}
    \min_{\mathbf{F}} \left[\left\|\mathbf{F}-\mathbf{X}\right\|_{F}^{2}+\lambda \cdot \max_{\mathbf{L}^{\prime}} \  \operatorname{tr}\left(\mathbf{F}^{\top} \mathbf{L}^{\prime} \mathbf{F}\right)\right] \quad \operatorname{s.t.} \quad \left\|\mathbf{L}^{\prime}-\widetilde{\mathbf{L}}\right\|_{F} \leq \varepsilon.
\end{equation}
Intuitively, the inner maximization on the Laplacian $\mathbf{L}'$ generates perturbations on the graph structure\footnote{Here we do not need exact graph structure perturbations as in graph adversarial attacks~\citep{zugner2018adversarial,zugner2019adversarial} but a virtual perturbation that could lead to small changes in the Laplacian.}, and enlarges the distance between the node representations of connected neighbors. 
Such maximization finds the worst case perturbations on the graph Laplacian that  hinders the global smoothness of $\mathbf{F}$ over the graph. Therefore, by training on those worse case Laplacian perturbations, one could obtain a robust graph signal denoising solution. Ideally, through solving Eq.~(\ref{adv gsd}), the smoothness of the node representations as well as the implicit denoising effect can be enhanced.

\subsection{Minimization of the Optimization Problem}
\label{Minimization of the Optimization Problem}
The min-max formulation in Eq.~(\ref{adv gsd}) also makes the adversarial graph signal denoising problem much harder to solve. Fortunately, unlike adversarial training~\citep{madry2017towards} where we need to first adopt PGD to solve the inner maximization problem before we solve the outer minimization problem, here inner maximization problem is simple and has a closed form solution. In other words, we do not need to add random perturbations on the graph structure at each training epoch and can find the largest perturbation which maximizes the inner adversarial loss function. Denote the perturbation as $\bm{\delta}$, and $\mathbf{L}'=\widetilde{\mathbf{L}} + \bm{\delta}$. Directly solving\footnote{More details on how to solve the inner maximization problem can be found in Appendix~\ref{appendix:how to solve the optimization problem}.} the inner maximization problem, we get  $\bm{\delta}=\varepsilon\nabla h(\bm{\delta})=\frac{\varepsilon\mathbf{F}\mathbf{F}^{\top}}{\left\|\mathbf{F}\mathbf{F}^{\top}\right\|_{F}}$. Plugging this solution into Eq.~(\ref{adv gsd}), we can rewrite the outer optimization problem as follows:
\begin{equation}
         \rho(\mathbf{F})=\min_{\mathbf{F}} \left[\left\|\mathbf{F}-\mathbf{X}\right\|_{F}^{2}+\lambda \max \operatorname{tr}\left(\mathbf{F}^{\top} \widetilde{\mathbf{L}} \mathbf{F}\right)+\lambda\varepsilon\operatorname{tr}\frac{\mathbf{F}^{\top}\mathbf{F}\mathbf{F}^{\top}\mathbf{F}}{\left\|\mathbf{F}\mathbf{F}^{\top}\right\|_{F}}\right].
\end{equation}
Taking the gradient of $\rho(\mathbf{F})$ to zero, we get the solution of the outer optimization problem as follows:
\begin{equation}\label{eq:advF}
    \mathbf{F} = \left(\mathbf{I}+\lambda\widetilde{\mathbf{L}}+\lambda\varepsilon\frac{\mathbf{F}\mathbf{F}^{\top}}{\left\|\mathbf{F}\mathbf{F}^{\top}\right\|_{F}}\right)^{-1}\mathbf{X}.
\end{equation}
Both sides of Eq.~(\ref{eq:advF}) contains $\mathbf{F}$, directly computing the solution is difficult. Note that in Eq.~(\ref{adv gsd}) we also require $\mathbf{F}$ to be close to $\mathbf{X}$, we can approximate Eq.~(\ref{eq:advF}) by replacing the $\mathbf{F}$ with $\mathbf{X}$ in the inverse matrix on the right hand side. With the Neumann series expansion of the inverse matrix, we get the final approximate solution as
\begin{equation}
\label{RNGC filter}
    \mathbf{H} \approx \frac{1}{\lambda+1}\sum_{s=0}^{S}\left[\frac{\lambda}{\lambda+1}\left(\widetilde{\bm{\mathcal{A}}}-\frac{\varepsilon\mathbf{X}\mathbf{X}^{\top}}{\left\|\mathbf{X}\mathbf{X}^{\top}\right\|_{F}}\right)\right]^{s}\mathbf{X}\mathbf{W}.
\end{equation}
The difference between Eq.~(\ref{RNGC filter}) and Eq.~(\ref{NGC filter}) is that there is one more term in Eq.~(\ref{RNGC filter}) derived from solving the inner optimization problem of Eq.~(\ref{adv gsd}). Based on this, we proposed our robust Neumann graph convolution (RNGC). 

\paragraph{Scalability.}{Although RNGC introduces extra computational burdens for large graphs due to the $\mathbf{X} \mathbf{X}^{\top}$ term, if the feature matrix is sparse, the extra computational effort is minimal as the $\mathbf{X} \mathbf{X}^{\top}$ term can also be sparse. For the scalability of RNGC on large graphs with dense feature matrix, we only compute the inner product of feature vectors ($\mathbf{X}_i, \mathbf{X}_{j|j\in\mathcal{N}_{i}}$) between adjacent neighbors like masked attention in GAT. Compared with \name, the additional computation cost is $\mathcal{O}(|\mathcal{E}|)$.} 

\section{Experiments}
\label{sec:experiments}
In this section, we conduct a comprehensive empirical study to understand the influence of different factors on the denoising effect of various models. To quantify the denoising effect, we test the model accuracy on noisy data on various GNN architectures and MLP for standard node classification tasks, where the noisy data is synthesized by mixing Gaussian noise with the original feature matrix. {We also synthesize noisy data by flipping individual feature with a small Bernoulli probability on three citation datasets with binary features.}

\begin{figure}[t]
    \begin{center}
        \includegraphics[width=0.325\textwidth]{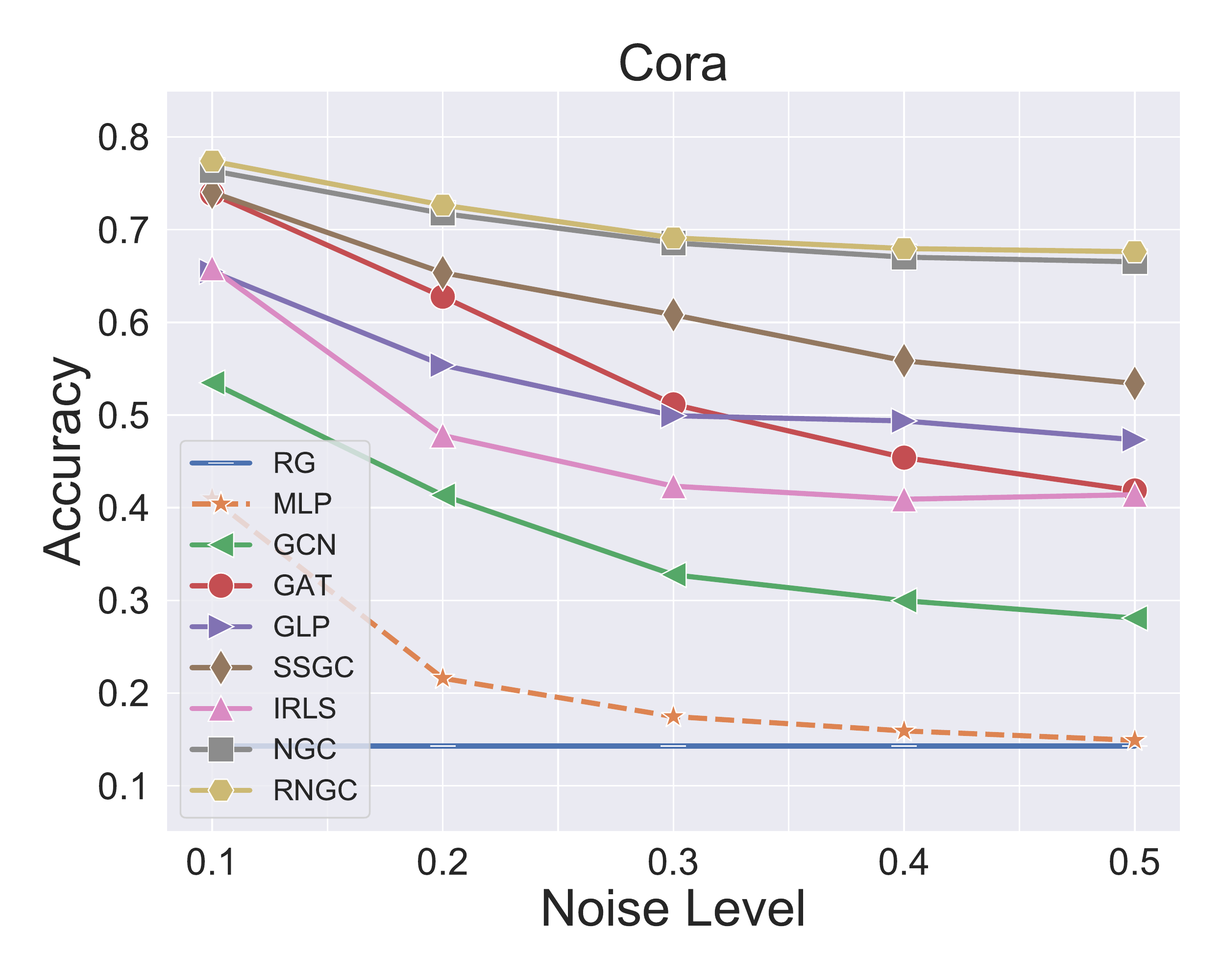}
        \includegraphics[width=0.325\textwidth]{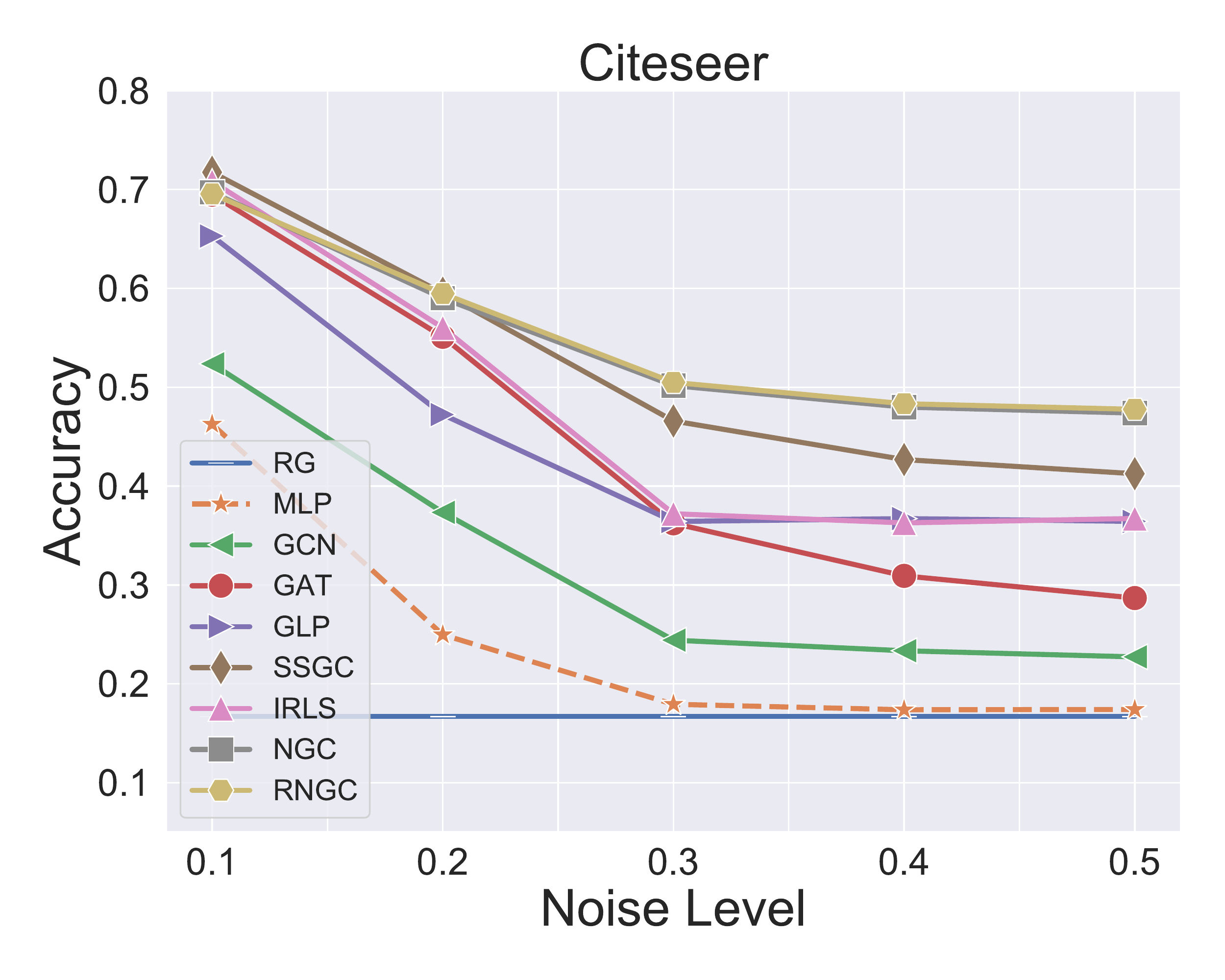}
        \includegraphics[width=0.325\textwidth]{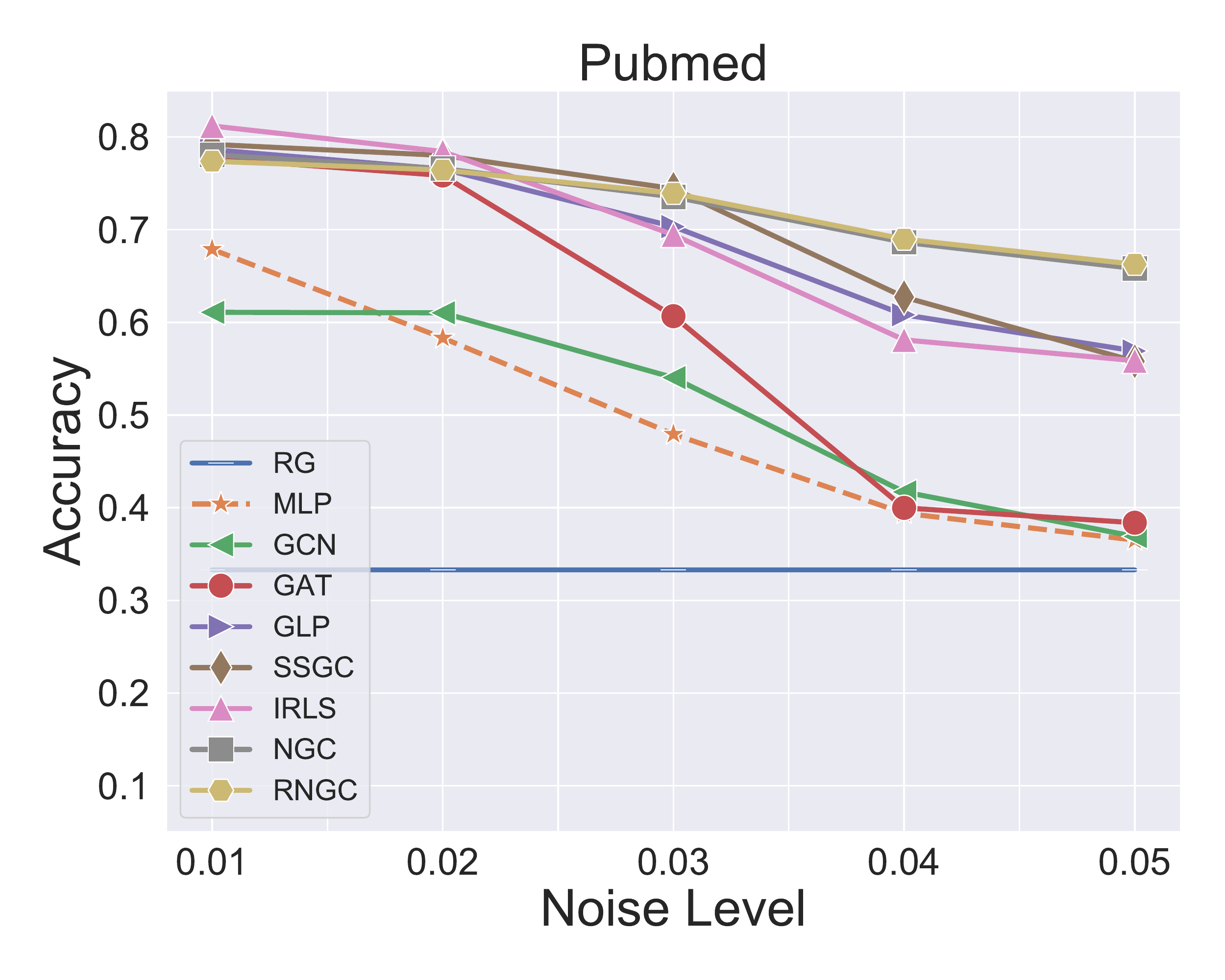}
    \end{center}
    \caption{Comparison of  classification accuracy v.s. noise level for semi-supervised node classification tasks. The noise level $\xi$ controls the magnitude of the Gaussian noise we add to the feature matrix: $\mathbf{X}+\xi\bm{\eta}$ where $\bm{\eta}$ is sampled from standard i.i.d., Gaussian distribution.}
    \label{fig:noise}
 \end{figure}

\subsection{Denoising Effectiveness Comparison of Various GNN Models}
In this section, we compare the denoising effectiveness of different GNN models through their test accuracy by training on the noisy feature matrix with Gaussian noise. 

\paragraph{Datasets.} In our experiments, we utilize three public citation network datasets Cora, Citeseer, and Pubmed~\citep{sen2008collective} which are homophily graphs for semi-supervised node classification. For the semi-supervised learning experimental setup, we follow the standard fixed splits employed in~\citep{yang2016revisiting}, with 20 nodes per class for training, 500 nodes for validation, and 1,000 nodes for testing. We also use four datasets: Cornell, Texas, Wisconsin, and Actor which are heterophily graphs for full-supervised node classification. For each dataset, we randomly split nodes into 60\%, 20\%, and 20\% for training, validation, and testing as suggested in \citep{pei2020geom}. Moreover, we utilize three large-scale graph datasets: Coauthor-CS, Coauthor-Phy~\citep{shchur2018pitfalls}, and ogbn-products~\citep{hu2020open} for evaluation. For Coauthor datasets, we split nodes into 60\%, 20\%, and 20\% for training, validation, and testing. For ogbn-products dataset, we follow the dataset split in OGB~\citep{hu2020open}.

\paragraph{Baselines.}
For the baselines, we consider graph neural networks derived from graph signal denoising, including GLP~\citep{li2019label}, S$^2$GC~\citep{zhu2021simple}, and IRLS~\citep{yang2021graph}; popular GNN architectures, such as GCN~\citep{kipf2017semi} and GAT~\citep{velivckovic2018graph}; and MLP which has no aggregation operation. 

\paragraph{Experimental Setup and Implementations.}
We assume that the original feature matrix is clean and do not have noise and we synthesize the noise from the standard Gaussian distribution and add them on the original feature matrix. By default, we apply row normalization for data after adding the Gaussian noise\footnote{We also perform an analysis on the effect of row normalization in noisy feature matrix in Appendix~\ref{appendix:row norm}.}, and train all the models based on these noisy feature matrix. For the hyper-parameters of each model, we follow the setting that reported in their original papers. To eliminate the effect of randomness, we repeat such experiment for 100 or 10 times and report the mean accuracy. Note that in each repeated run, we add different Gaussian noises. While for the same run, we apply the same noisy feature matrix for training all the models. For our \name and R\name model, the hyper-parameter details can be found in Appendix~\ref{appendix:hyperparameter}.

\begin{table}[t]
\caption{Summary of results (10 runs) on heterophily graphs in terms of classification accuracy (\%)}
\label{tab:heterophily}
\normalsize
\setlength{\tabcolsep}{3mm}
\begin{center}
\scalebox{0.75}{\begin{tabular}{ccccccccc}
\toprule
\multirow{2}{*}{Noise Level} & \multicolumn{2}{c}{Cornell} & \multicolumn{2}{c}{Texas} & \multicolumn{2}{c}{Wisconsin} & \multicolumn{2}{c}{Actor} \\
\cmidrule(r){2-3} \cmidrule(r){4-5} \cmidrule(r){6-7} \cmidrule(r){8-9} 
& 0.01 &  1  &  0.01  &  1  & 0.01 & 1   &  0.01  &  1 \\
\midrule
MLP & 69.7\scriptsize{$\pm$8.6} & 55.3\scriptsize{$\pm$7.6} & 69.7\scriptsize{$\pm$8.6} & 55.3\scriptsize{$\pm$7.6} & \textbf{78.6\scriptsize{$\pm$6.5}}  & 44.6\scriptsize{$\pm$6.2} & 33.3\scriptsize{$\pm$1.1} & \textbf{25.1\scriptsize{$\pm$1.0}}  \\
GCN & 56.9\scriptsize{$\pm$8.4} & 51.7\scriptsize{$\pm$14.9} & 56.6\scriptsize{$\pm$8.1} & 52.5\scriptsize{$\pm$11.7} & 48.0\scriptsize{$\pm$6.1}  & 41.2\scriptsize{$\pm$9.0} & 26.4\scriptsize{$\pm$1.0} & 23.8\scriptsize{$\pm$3.0}  \\
GAT & 55.8\scriptsize{$\pm$8.9} & 55.0\scriptsize{$\pm$7.5} & 56.4\scriptsize{$\pm$8.1} & 54.7\scriptsize{$\pm$8.2} & 53.4\scriptsize{$\pm$7.2}  & \textbf{48.2\scriptsize{$\pm$7.2}} & 27.3\scriptsize{$\pm$1.2} & 24.3\scriptsize{$\pm$0.7}  \\
GLP & 65.3\scriptsize{$\pm$8.6} & 54.2\scriptsize{$\pm$7.6} & 60.0\scriptsize{$\pm$9.3} & 52.8\scriptsize{$\pm$8.0} & 59.0\scriptsize{$\pm$5.4}  & 42.6\scriptsize{$\pm$5.4} & 31.0\scriptsize{$\pm$1.3} & \textbf{25.1\scriptsize{$\pm$0.8}}  \\
S$^2$GC & 60.6\scriptsize{$\pm$9.3} & 48.6\scriptsize{$\pm$10.4} & 56.4\scriptsize{$\pm$7.2} & 50.3\scriptsize{$\pm$8.0} & 47.4\scriptsize{$\pm$4.5}  & 37.2\scriptsize{$\pm$3.7} & 27.2\scriptsize{$\pm$1.1} & 23.5\scriptsize{$\pm$1.3}  \\
IRLS & 48.1\scriptsize{$\pm$8.5} & 46.7\scriptsize{$\pm$6.2} & 65.6\scriptsize{$\pm$8.4} & 42.8\scriptsize{$\pm$15.9} & 65.2\scriptsize{$\pm$6.0}  & 37.4\scriptsize{$\pm$8.5} & \textbf{36.1\scriptsize{$\pm$0.9}} & 21.5\scriptsize{$\pm$4.1}  \\
NGC & \textbf{72.8\scriptsize{$\pm$8.7}} & \textbf{56.4\scriptsize{$\pm$8.1}} & \textbf{73.9\scriptsize{$\pm$6.9}} & \textbf{56.4\scriptsize{$\pm$8.1}} & 74.8\scriptsize{$\pm$6.8}  & \textbf{46.8\scriptsize{$\pm$6.6}} & 34.0\scriptsize{$\pm$1.6} & \textbf{25.1\scriptsize{$\pm$1.0}}  \\
RNGC & \textbf{75.8\scriptsize{$\pm$7.9}} & \textbf{56.4\scriptsize{$\pm$8.1}} & \textbf{74.2\scriptsize{$\pm$6.1}} & \textbf{56.4\scriptsize{$\pm$8.1}} & \textbf{76.4\scriptsize{$\pm$5.3}}  & \textbf{46.8\scriptsize{$\pm$6.6}} & \textbf{34.3\scriptsize{$\pm$1.6}} & \textbf{25.1\scriptsize{$\pm$1.0}}  \\
\bottomrule
\end{tabular}}
\end{center}
\end{table}
\begin{table}[t]
 \begin{minipage}[t]{0.6\textwidth}
 \normalsize    \makeatletter\def\@captype{table}\makeatother\parbox{7.5cm}{
 \vskip -0.15in
 \caption{Summary of results (10 runs) on Coauthor-CS and Coauthor-Phy in terms of accuracy (\%)}\label{tab:coauthor}}
    \begin{center}
    \setlength{\tabcolsep}{3mm}
        \scalebox{0.75}{\begin{tabular}{ccccc}
        \toprule
        \multirow{2}{*}{Noise Level} & \multicolumn{2}{c}{Coauthor-CS} & \multicolumn{2}{c}{Coauthor-Phy}  \\
        \cmidrule(r){2-3} \cmidrule(r){4-5} 
        & 0.1 &  1  &  0.1  &  1    \\
        \midrule
        MLP & 82.5\scriptsize{$\pm$1.8} & 22.3\scriptsize{$\pm$0.1} & 81.6\scriptsize{$\pm$8.1} & 47.0\scriptsize{$\pm$10.0}  \\
        GCN & 87.3\scriptsize{$\pm$0.5} & 61.3\scriptsize{$\pm$14.3} & 94.2\scriptsize{$\pm$0.4} & 78.6\scriptsize{$\pm$10.6}   \\
        GAT & 86.8\scriptsize{$\pm$3.6} & 57.9\scriptsize{$\pm$20.2} & 94.0\scriptsize{$\pm$0.4} & 63.7\scriptsize{$\pm$16.7}  \\
        GLP & 91.3\scriptsize{$\pm$0.4} & 52.4\scriptsize{$\pm$17.3} & 93.3\scriptsize{$\pm$2.5} & 81.3\scriptsize{$\pm$10.6}  \\
        S$^2$GC & 86.1\scriptsize{$\pm$0.2} & 79.6\scriptsize{$\pm$10.2} & 92.6\scriptsize{$\pm$1.3} & 89.4\scriptsize{$\pm$4.3}  \\
        IRLS & 78.8\scriptsize{$\pm$5.1} & 62.1\scriptsize{$\pm$17.8} & 89.2\scriptsize{$\pm$3.4} & 87.0\scriptsize{$\pm$4.5}  \\
        NGC & \textbf{95.3\scriptsize{$\pm$}0.2} & \textbf{87.1\scriptsize{$\pm$}3.1} & \textbf{95.7\scriptsize{$\pm$}0.2} & \textbf{93.1\scriptsize{$\pm$}1.4} \\
        RNGC & \textbf{95.4\scriptsize{$\pm$}0.2} & \textbf{87.8\scriptsize{$\pm$}1.5} & \textbf{95.7\scriptsize{$\pm$}0.2} & \textbf{93.6\scriptsize{$\pm$}0.8}  \\
        \bottomrule
        \end{tabular}}
        \end{center}
  \end{minipage}
  \begin{minipage}[t]{0.3\textwidth}
  \normalsize
\makeatletter\def\@captype{table}\makeatother\parbox{5cm}{\caption{Summary of results (10 runs) on ogbn-products in terms of accuracy (\%)}\label{tab:ogb}}
        \setlength{\tabcolsep}{3mm}
        \begin{center}
         \scalebox{0.75}{\begin{tabular}{ccc}
            \toprule
            \multirow{2}{*}{Noise Level} & \multicolumn{2}{c}{ogbn-products}  \\
            \cmidrule(r){2-3}
             & 0.1 & 1   \\
            \midrule
            MLP & 59.68\scriptsize{$\pm$0.16}  & 38.08\scriptsize{$\pm$0.10} \\
            GCN & 75.60\scriptsize{$\pm$0.19} & 72.76\scriptsize{$\pm$0.20}  \\
            S$^2$GC &  74.95\scriptsize{$\pm$0.13}  & 63.17\scriptsize{$\pm$0.12}  \\
            NGC & \textbf{77.56\scriptsize{$\pm$0.15}} & \textbf{73.36\scriptsize{$\pm$0.11}}  \\
            RNGC & \textbf{77.54\scriptsize{$\pm$}0.15}  & \textbf{73.66\scriptsize{$\pm$}0.13}  \\
            \bottomrule
            \end{tabular}}
        \end{center}
   \end{minipage}
 \end{table}
\paragraph{Results on Supervised Node Classification.}
Figure~\ref{fig:noise} illustrates the comparison of classification accuracy against the various noise levels for semi-supervised node classification tasks. 
The noise level $\xi$ controls the magnitude of the Gaussian noise we add to the feature matrix: $\mathbf{X}+\xi\bm{\eta}$ where $\bm{\eta}$ is sampled from standard i.i.d., Gaussian distribution. For Cora and Citeseer, we test $\xi \in \{0.1, 0.2, 0.3, 0.4, 0.5\}$ and for Pubmed, we test $\xi \in \{0.01, 0.02, 0.03, 0.04, 0.05\}$. From Figure~\ref{fig:noise}, we can observe that the test accuracy of MLP is close to randomly guessing (RG) when the noise level is relatively large. This implies the weak denoising effect of MLP models. For shallow GNN models, such as GCN and GAT (which usually contain 2 layers), their denoising performance is limited especially on Pubmed since they do not aggregate information (features and noise) from higher-order neighbors. For models with deep layers{\footnote{We also perform an analysis on the denoising effect of depth in \name and R\name in Appendix~\ref{appendix:depth analysis}.}}, such as IRLS ($\geq 8$ layers), the denoising performance is much better compared to shallow models. Lastly, our \name and R\name model with 16 layers ($S=16$) achieve significantly better denoising performance compared with other baseline methods, which backup our theoretical analyses. In most cases, \name and R\name achieve very similar denoising performance but in general, R\name still slightly outperforms \name, suggesting that we indeed gain more benefits by solving the adversarial graph denoising problem. 

{Table~\ref{tab:heterophily} reports the comparison of classification accuracy against the various noise levels for full-supervised node classification tasks on heterophily graphs. The first- and second-highest accuracies are highlighted in bold. For these datasets, we test $\xi \in \{0.01, 1\}$. From Table~\ref{tab:heterophily}, we can observe that MLP is better than most GNN models in most cases due to the heterophily properties of these graphs. However, our proposed R\name achieves significantly better or matches denoising performance compared with other baseline methods, which demonstrates the superiority of our R\name.}

{For ogbn-products, we only choose MLP, GCN, and S$^2$GC as baselines, since the results are sensitive concerning model size and various tricks from the OGB leaderboard. For fair comparison, the size of parameters for these baselines and R\name is the same. We also use full-batch training for the baselines and our model. Table~\ref{tab:coauthor} and \ref{tab:ogb} report the comparison of classification accuracy against the various noise levels for full-supervised node classification tasks on large-scale graphs. The first- and second-highest accuracies are highlighted in bold. For these datasets, we test $\xi \in \{0.1, 1\}$. Compared with the above small datasets, the node degree on these three datasets is larger, which means they have better connectivity. From Table~\ref{tab:coauthor} and \ref{tab:ogb}, we can observe that the test accuracy of MLP is far lower than GCN and R\name. This implies the weak denoising effect of MLP. The test accuracy of GCN is slightly smaller than R\name on these datasets since they are well-connected and have a large graph size and we can achieve a good denoising performance with shallow-layer GNN models. For the scalability of R\name on large graphs such as ogbn-products, we use the acceleration method mentioned in Sec.~\ref{Minimization of the Optimization Problem}.}
 \begin{table}[t]
 \vskip -0.1in
 \begin{minipage}[t]{0.6\textwidth}
 \normalsize
    \makeatletter\def\@captype{table}\makeatother\parbox{8cm}{\caption{Denoising performance over 100 runs against flipping perturbation}\label{tab:flip}}
    \begin{center}
    \setlength{\tabcolsep}{1mm}
        \scalebox{0.75}{\begin{tabular}{ccccccccccc}
        \toprule
        \multirow{2}{*}{Flipping probability} & \multicolumn{3}{c}{Cora} & \multicolumn{3}{c}{Citeseer} & \multicolumn{3}{c}{Pubmed}   \\
        \cmidrule(r){2-4} \cmidrule(r){5-7}  \cmidrule(r){8-10} 
        & 0.1 &  0.2 & 0.4 & 0.1 &  0.2 & 0.4 & 0.1 &  0.2 & 0.4 \\
        \midrule
        MLP & 21.2 & 21.1 & 23.3 & 19.3  & 18.9 & 18.9 & 38.0 & 39.0 & 40.6 \\
        GCN & 22.9 & 19.0 & 19.0 & 18.6 & 18.6 & 18.5 & 37.8 & 38.1 & 37.6 \\
        GAT & 70.1 & 65.6 & 60.0 & 45.3 & 39.3 & 26.0 & 43.3 & 49.5 & 60.0 \\
        GLP & 32.3 & 30.8 & 29.0 & 19.7 & 18.9 & 18.8 & 42.1 & 41.5 & 40.7 \\
        S$^2$GC & 75.0 & 71.5 & 63.8 & 49.9 & 46.4 & 43.4 & 50.4 & 60.2 & 69.3 \\
        IRLS & 66.4 & 61.0 & 54.7 & 50.3 & 45.9 & 43.8 & 51.4 & 60.0 & 69.0 \\
        NGC & \textbf{77.5} & \textbf{75.3} & \textbf{65.7} & \textbf{54.9} & \textbf{51.9} & \textbf{48.5} & \textbf{53.0} & \textbf{62.3} & \textbf{70.4} \\
        RNGC & \textbf{77.6} & \textbf{75.2} & \textbf{72.8} & \textbf{55.0} & \textbf{51.8} & \textbf{48.7} & \textbf{54.3} & \textbf{63.9} & \textbf{71.6} \\
        \bottomrule
        \end{tabular}}
        \end{center}
  \end{minipage}
  \begin{minipage}[t]{0.35\textwidth}
  \normalsize
     \makeatletter\def\@captype{table}\makeatother\parbox{4.8cm}{\caption{Defense performance over 100 runs against structure attck}\label{tab:attack}}
        \setlength{\tabcolsep}{1mm}
        \begin{center}
         \scalebox{0.75}{\begin{tabular}{lccc}
            \hline
            Model	& Cora & Citeseer & Pubmed  \\
            \hline\hline
            GCN & 47.53 & 56.94 & 75.50 \\
            GAT & 54.78 & 61.85 & 65.41 \\
            RobustGCN & 50.51 & 55.35 & 67.95 \\
            GCN-Jaccard & 60.82 & 59.89 & 83.66 \\
            GCN-SVD & 52.06 & 57.18 & 82.72 \\
            S$^2$GC & 51.60 & 54.11 & 64.04 \\
            RNGC & \textbf{63.16} & \textbf{65.64} & \textbf{84.04} \\
            \hline
        \end{tabular}}
        \end{center}
   \end{minipage}
 \end{table}
\subsection{Denoising Performance on Feature Flipping Perturbation}
In this section, we compare the denoising effectiveness of different models through their test accuracy by training on the noisy feature matrix which is perturbated through flipping the individual feature with a small Bernoulli probability on three citation datasets.
\paragraph{Setting and Results.} We flip the individual feature on three citation datasets: Cora, Citeseer, and Pubmed as the noise. And we compare the denoising performance of R\name with MLP and GCN. From Table~\ref{tab:flip}, we can observe that the denoising performance of R\name is much better than baselines when the flip probability is 0.4. In fact, the added perturbations by flipping the individual feature approximately follow a Bernoulli distribution, which is also a Sub-Gaussian distribution. The results verify our theoretical analysis further.

\subsection{Defense Performance of R\name against Graph Structure Attack}
Although we do not perform actual graph structure perturbations as in graph adversarial attacks~\citep{zugner2018adversarial,zugner2019adversarial} but a virtual perturbation in the Laplacian. Therefore, it's not clear how much perturbations on the Laplacian correspond to the actual perturbations on graph structure. Nevertheless, we still conduct the experiments of R\name against graph structure meta-attack where the ptb rate is 25\%. As shown in the Table~\ref{tab:attack}, our R\name model outperforms than GCN, GAT, RobustGCN \citep{zugner2019robustgcn}, GCN-Jaccard \citep{wu2019adversarial}, GCN-SVD \citep{entezari2020all}, and S$^2$GC on Cora, Citeseer, and Pubmed.

\section{Related Work}
\paragraph{Implicit Denoising in GNNs.}
Existing graph denoising works are mainly based on the graph smoothing technique~\citep{chen2014signal,zhou2021graph}. It is well known that GNNs can increase the smoothness of node features through aggregating information from neighbors, thus the influence from noisy features can be counteracted in GNN's output. Some recent GNN models are derived from the perspective of signal denoising, such as S$^2$GC~\citep{zhu2021simple}, GLP~\citep{li2019label}, and IRLS~\citep{yang2021graph}. Moreover, Ma et al.~\citep{ma2021unified} builds the connection between signal denoising and existing popular GNNs by formulating message passing as a process of solving the GSD problem. The relationship between GSD and GCN can be briefly illustrated as follows~\citep{ma2021unified}. This suggests a possibility for us to understand the behavior of GNNs through the lens of signal denoising. To our best knowledge, we are the first to offer a theoretical analysis to understand the denoising effect of GNNs. Besides, there is a recent work~\citep{zhang2022graphless} to conduct the empirical study of the denoising effect in GNNs. In this work, we perform an extensive analysis to understand the denoising effect of GNNs from both theoretical and experimental perspectives.
\paragraph{Smoothing and Over-smoothing.}One key principle of GNNs is to improve the smoothness of node representations. But stacking graph layers can lead to over-smoothing~\citep{li2018deeper}, where the node representations can not be distinguishable. There are some recent works that have been proposed to address over-smoothing such as JKnet~\citep{xu2018representation}, GCNII~\citep{chen2020simple}, and RevGNN-Deep~\citep{pmlr-v139-li21o}. They add the output of shallow layers to the final layers with a residual-style design. In this work, we will show smoothing can help the denoising effect of GNNs.

\section{Conclusion}
Our work conducts a comprehensive study on the implicit denoising effect of graph neural networks. We theoretical show that the denoising effect of GNNs are largely influenced by the connectivity and the size of the graph structure, as well as the GNN architectures. Motivated by our analysis, we also propose a robust graph convolution model by solving the robust graph signal denoising problem which enhances the smoothness of node representations and the implicit denoising effect.

\newpage

\normalem
\bibliography{cite}
\bibliographystyle{iclr2023_conference}

\newpage

\appendix
\section{The details on how to solve the inner maximization problem in Sec.~\ref{Minimization of the Optimization Problem}}
\label{appendix:how to solve the optimization problem}
Different from the non-concave inner maximization problem in the adversarial attack, our inner maximization problem is indeed a convex optimization problem. Hence, we do not need to add random perturbations on the graph structure at each training epoch and can find the largest perturbation which maximizes the inner adversarial loss function. Denote the perturbation as $\bm{\delta}$, and $\mathbf{L}'=\widetilde{\mathbf{L}} + \bm{\delta}$. We can rewrite the inner maximization problem as
\begin{equation}
    \max_{\mathbf{L}^{\prime}} \operatorname{tr}\left(\mathbf{F}^{\top} \mathbf{L}^{\prime} \mathbf{F}\right) =\langle\widetilde{\mathbf{L}}, \mathbf{F}^{\top}\mathbf{F}\rangle + \max_{\bm{\delta}}\langle\bm{\delta}, \mathbf{F}^{\top}\mathbf{F}\rangle \quad \operatorname{s.t.} \quad \left\|\bm{\delta}\right\|_{F} \leq \varepsilon.
\end{equation}
We denote $h(\bm{\delta})=\langle\bm{\delta}, \mathbf{F}^{\top}\mathbf{F}\rangle$. Obviously, $h(\bm{\delta})$ reaches the largest value when $\bm{\delta}$ has the same direction with the gradient of $h(\bm{\delta})$, e.g. $\bm{\delta}=\varepsilon\nabla h(\bm{\delta})=\frac{\varepsilon\mathbf{F}\mathbf{F}^{\top}}{\left\|\mathbf{F}\mathbf{F}^{\top}\right\|_{F}}$, which is illustrated in Fig.~\ref{figure:adv inner problem}.
\begin{figure}[htbp]
 \begin{center}
\includegraphics[width=.45\columnwidth]{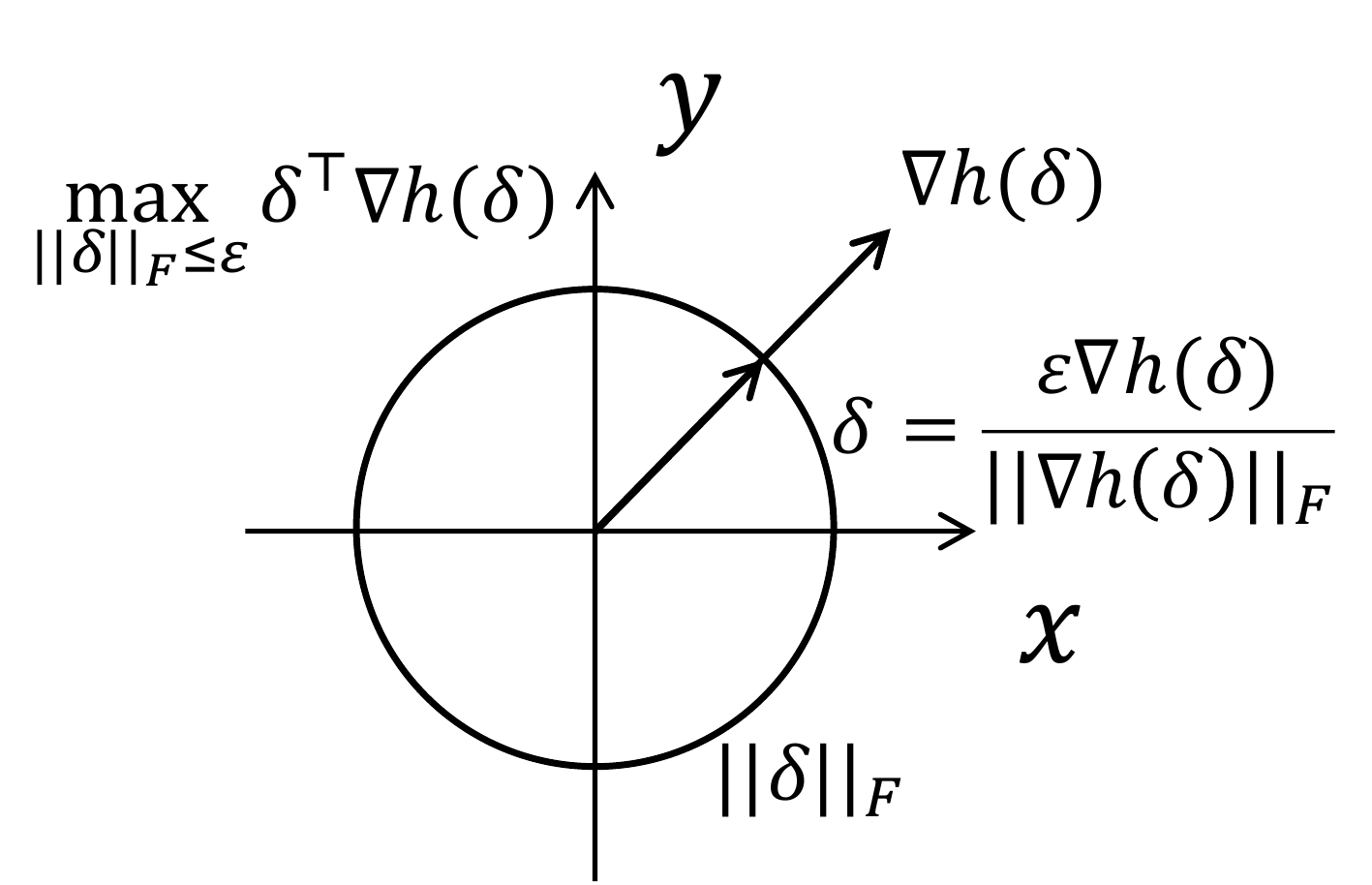}
 \end{center}
\caption{The illustration of the inner maximization problem. The adversarial loss function reaches the largest value when the direction of $\bm{\delta}$ is the same with $\nabla h(\bm{\delta})$}
\label{figure:adv inner problem}
\end{figure}

\section{Additional Details on the Neumann Series}
\label{appendix:neumann series}
We provide additional details and derivations on how to obtain the Neumann Series which leads to our Neumann Graph Convolution (NGC) method. Before we derive the Neumann Series, we first introduce the following lemmas which are crucial to the derivation of the Neumann Series.

\begin{lemma} \rm {\textbf{(Gelfand formula)~\citep{bhatia2013matrix}} } 
\label{Gelfand formula}
\emph{Given any matrix norm $\||\cdot|\|$, then $\rho(\mathbf{A})=\lim\limits_{k \rightarrow \infty}\||\mathbf{A}^{k}|\|^{1 / k}=\inf\limits_{k \geq 1}\||\mathbf{A}^k|\|^{1/k}\leq \||\mathbf{A}|\|$}.
\end{lemma}
Lemma~\ref{Gelfand formula} describes the relationship between the spectral radius of a matrix and its matrix norm, $i.e.$ $\rho(\mathbf{A})=\lim\limits_{k \rightarrow \infty}\||\mathbf{A}^{k}|\|^{1 / k}$.

\begin{lemma}
\label{convergence of neumann series}
Let $\mathbf{A} \in \mathbb{C}^{n \times n}$, the spectral radius $\rho(\mathbf{A})=\max (\operatorname{abs}(\operatorname{spec}(\mathbf{A})))$, if $\rho(\mathbf{A})<1$, then $\sum_{k=0}^{\infty} \mathbf{A}^{k}$ converges to $(\mathbf{I}-\mathbf{A})^{-1}$.
\end{lemma}
\begin{proof}
We first prove that $(\mathbf{I}-\mathbf{A})^{-1}$ exists as follows: Based on the definition of eigenvalues of $\mathbf{A}$, we have $|\lambda \mathbf{I} - \mathbf{A}| = 0$ and the solution is the eigenvalue of $\mathbf{A}$. Since $\rho(\mathbf{A}) < 1$, if $\lambda \geq 1$, then $|\lambda \mathbf{I} - \mathbf{A}| \neq 0$, so $|\mathbf{I} - \mathbf{A}| \neq 0$, which means $(\mathbf{I}-\mathbf{A})^{-1}$ exists.
 
Since $\rho(\mathbf{A}) < 1$ and by Lemma~\ref{Gelfand formula}, we have $\lim\limits_{k \rightarrow \infty}\||\mathbf{A}^{k}|\|=\rho(\mathbf{A})^k=0$. Let $\mathbf{S}_k$ = $\mathbf{A}^0 + \mathbf{A}^1 + \cdots + \mathbf{A}^k$, then we have
\begin{equation*}
    \begin{split}
        \lim_{k \rightarrow \infty}(\mathbf{S}^k-\mathbf{A}\mathbf{S}^k) &= \lim\limits_{k \rightarrow \infty}(\mathbf{I}-\mathbf{A})\mathbf{S}^k \\
        &=\lim_{k \rightarrow \infty}(\mathbf{I}-\mathbf{A}^{k+1}) \\
        &= \mathbf{I}
    \end{split}
\end{equation*}
Since $(\mathbf{I}-\mathbf{A})^{-1}$ exists, so we have $(\mathbf{I}-\mathbf{A})\lim\limits_{k \rightarrow \infty}\mathbf{S}^k=\mathbf{I}$, and $\lim\limits_{k \rightarrow \infty}\mathbf{S}^k=(\mathbf{I}-\mathbf{A})^{-1}$, which finishes the proof.
\end{proof}

Lemma~\ref{convergence of neumann series} describes the convergence of Neumann Series and the condition to get the convergence.
\begin{lemma} \rm {\textbf{(Gerschgorin Disc)~\citep{bhatia2013matrix}}} 
\label{Gerschgorin Disc}
\emph{Let $\mathbf{A} \in \mathbb{C}^{n \times n}$, with entries $a_{ij}$. For any eigenvalue $\lambda$, there exits $i$ and the corresponding Gerschgorin disc $D\left(a_{i i}, R_{i}\right) \subseteq \mathbb{C}$ such that $\lambda$ lies in this disc, i.e.}
\begin{equation*}
    |\lambda-a_{ii}| \leq \sum_{j\neq i}^{n}|a_{ij}|.
\end{equation*}
\end{lemma}
Lemma~\ref{Gerschgorin Disc} describes the estimated range of eigenvalues. Now we start to derive the Neumann Series expansion of the solution of GSD as follows.

\begin{lemma}
\label{Neumann derivation}
Let $\mathbf{A} \in\{0,1\}^{n \times n}$ be the adjacency matrix of a graph and $\widetilde{\bm{\mathcal{A}}}=\widetilde{\mathbf{D}}^{-\frac{1}{2}} \widetilde{\mathbf{A}} \widetilde{\mathbf{D}}^{-\frac{1}{2}}$ or $\widetilde{\bm{\mathcal{A}}} = \widetilde{\mathbf{D}}^{-1} \widetilde{\mathbf{A}}$, then
\begin{equation*}
    (\mathbf{I}-\frac{\lambda}{\lambda+1}\widetilde{\bm{\mathcal{A}}})^{-1}=\sum_{k=0}^{\infty} \left(\frac{\lambda}{\lambda+1}\widetilde{\bm{\mathcal{A}}}\right)^{k}.
\end{equation*}
\end{lemma}
\begin{proof}
We first prove that $\rho(\widetilde{\bm{\mathcal{A}}})\leq1$ where $\widetilde{\bm{\mathcal{A}}}=\widetilde{\mathbf{D}}^{-\frac{1}{2}} \widetilde{\mathbf{A}} \widetilde{\mathbf{D}}^{-\frac{1}{2}}$.
Let $\lambda$ be the eigenvalue of $\widetilde{\bm{\mathcal{A}}}$, and $\mathbf{v}$ be the corresponding eigenvector. Then we have
\begin{equation*}
\begin{split}
    \left(\widetilde{\mathbf{D}}^{-\frac{1}{2}} \widetilde{\mathbf{A}} \widetilde{\mathbf{D}}^{-\frac{1}{2}}\right)\mathbf{v}=\lambda\mathbf{v} &\Longrightarrow \widetilde{\mathbf{D}}^{-\frac{1}{2}} \left(\widetilde{\mathbf{D}}^{-\frac{1}{2}} \widetilde{\mathbf{A}} \widetilde{\mathbf{D}}^{-\frac{1}{2}}\right)\mathbf{v}=\lambda\widetilde{\mathbf{D}}^{-\frac{1}{2}}\mathbf{v}\\
    &\Longrightarrow \left(\widetilde{\mathbf{D}}^{-1} \widetilde{\mathbf{A}}\right) \widetilde{\mathbf{D}}^{-\frac{1}{2}}\mathbf{v} = \lambda\widetilde{\mathbf{D}}^{-\frac{1}{2}}\mathbf{v},
\end{split}
\end{equation*}
which means $(\lambda, \widetilde{\mathbf{D}}^{-\frac{1}{2}}\mathbf{v})$ is the eigen-pair of $\widetilde{\mathbf{D}}^{-1}\mathbf{A}$. By Lemma~\ref{Gerschgorin Disc}, there exists $i$, such that
\begin{equation*}
\begin{split}
    &\left|\lambda-\left(\widetilde{\mathbf{D}}^{-1}\widetilde{\mathbf{A}}\right)_{ii}\right|\leq \sum_{j\neq i}\left|\left(\widetilde{\mathbf{D}}^{-1}\widetilde{\mathbf{A}}\right)_{ij}\right|\\
    &\Longrightarrow \left(\widetilde{\mathbf{D}}^{-1}\widetilde{\mathbf{A}}\right)_{ii} - \sum_{j\neq i}\left|\left(\widetilde{\mathbf{D}}^{-1}\widetilde{\mathbf{A}}\right)_{ij}\right| \leq 
    \lambda \leq \left(\widetilde{\mathbf{D}}^{-1}\widetilde{\mathbf{A}}\right)_{ii} + \sum_{j\neq i}\left|\left(\widetilde{\mathbf{D}}^{-1}\widetilde{\mathbf{A}}\right)_{ij}\right|.
\end{split}
\end{equation*}
Since $\left(\widetilde{\mathbf{D}}^{-1}\widetilde{\mathbf{A}}\right)_{ij}>0$ and $\sum_{j}\left|\left(\widetilde{\mathbf{D}}^{-1}\widetilde{\mathbf{A}}\right)_{ij}\right|=\sum_{j}\left(\widetilde{\mathbf{D}}^{-1}\widetilde{\mathbf{A}}\right)_{ij}=1$, obviously
\begin{equation*}
    -1<\left(\widetilde{\mathbf{D}}^{-1}\widetilde{\mathbf{A}}\right)_{ii} - \sum_{j\neq i}\left|\left(\widetilde{\mathbf{D}}^{-1}\widetilde{\mathbf{A}}\right)_{ij}\right| \leq 
    \lambda \leq \left(\widetilde{\mathbf{D}}^{-1}\widetilde{\mathbf{A}}\right)_{ii} + \sum_{j\neq i}\left|\left(\widetilde{\mathbf{D}}^{-1}\widetilde{\mathbf{A}}\right)_{ij}\right|=1.
\end{equation*}
So if $\widetilde{\bm{\mathcal{A}}}=\widetilde{\mathbf{D}}^{-\frac{1}{2}} \widetilde{\mathbf{A}} \widetilde{\mathbf{D}}^{-\frac{1}{2}}$, we have $\rho(\widetilde{\bm{\mathcal{A}}})\leq1$. When $\widetilde{\bm{\mathcal{A}}}=\widetilde{\mathbf{D}}^{-1}\widetilde{\mathbf{A}}$, we denote $(\lambda, \mathbf{v})$ as the eigen-pair of $\mathbf{\widetilde{\mathbf{D}}^{-1}\mathbf{A}}$. Similarly, by Lemma~\ref{Gerschgorin Disc}, there exists $i$, such that
\begin{equation*}
\begin{split}
    &\left|\lambda-\left(\widetilde{\mathbf{D}}^{-1}\widetilde{\mathbf{A}}\right)_{ii}\right|\leq \sum_{j\neq i}\left|\left(\widetilde{\mathbf{D}}^{-1}\widetilde{\mathbf{A}}\right)_{ij}\right|\\
    &\Longrightarrow \left(\widetilde{\mathbf{D}}^{-1}\widetilde{\mathbf{A}}\right)_{ii} - \sum_{j\neq i}\left|\left(\widetilde{\mathbf{D}}^{-1}\widetilde{\mathbf{A}}\right)_{ij}\right| \leq 
    \lambda \leq \left(\widetilde{\mathbf{D}}^{-1}\widetilde{\mathbf{A}}\right)_{ii} + \sum_{j\neq i}\left|\left(\widetilde{\mathbf{D}}^{-1}\widetilde{\mathbf{A}}\right)_{ij}\right|.
\end{split}
\end{equation*}
Obviously, we can get the same conclusion for $\widetilde{\bm{\mathcal{A}}}=\widetilde{\mathbf{D}}^{-1}\widetilde{\mathbf{A}}$. So it is true for $\rho\left(\frac{\lambda}{\lambda+1}\widetilde{\bm{\mathcal{A}}}\right) \leq \frac{\lambda}{\lambda+1}<1$ By Lemma~\ref{convergence of neumann series}, we get the result $(\mathbf{I}-\frac{\lambda}{\lambda+1}\widetilde{\bm{\mathcal{A}}})^{-1}=\sum_{k=0}^{\infty} \left(\frac{\lambda}{\lambda+1}\widetilde{\bm{\mathcal{A}}}\right)^{k}$, which finishes the proof.
\end{proof}

By Lemma~\ref{Neumann derivation}, we approximate the inverse matrix $(\mathbf{I}+\lambda \tilde{\mathbf{L}})^{-1}$ up to $S$-th order with
\begin{equation*}
    \left(\mathbf{I}+\lambda \widetilde{\mathbf{L}}\right)^{-1} =\frac{1}{\lambda+1}\left(\mathbf{I}-\frac{\lambda}{\lambda+1}\widetilde{\bm{\mathcal{A}}}\right)^{-1}\approx\frac{1}{\lambda+1}\sum_{s=0}^{S} \left(\frac{\lambda}{\lambda+1}\widetilde{\bm{\mathcal{A}}}\right)^{s}.
\end{equation*}

\section{The Row Summation of the Neumann Series}
\label{appendix:row sum}
We provide the derivations of the row sum of $\widetilde{\bm{\mathcal{A}}}_{S}$ in this section. Before we derive the row summation of $\widetilde{\bm{\mathcal{A}}}_{S}$, we first derive the row summation of $\widetilde{\bm{\mathcal{A}}}^{k}$.
\begin{lemma}
\label{transition matrix}
Consider a probability matrix $\mathbf{P} \in \mathbb{R}^{n \times n}$, where $\mathbf{P}_{ij}\geq0$. Besides, for all $i$, we have $\sum_{j=1}^n\mathbf{P}_{ij}=1$. Then for any $s \in \mathbb{Z}_{+}$, we have $\sum_{j=1}^n\mathbf{P}^s_{ij}=1$,
\end{lemma}
\begin{proof}
We give a proof by induction on $k$.\\
\textbf{Base case:} When $k=1$, the case is true.\\
\textbf{Inductive step:} Assume the induction hypothesis that for a particular $k$, the single case n = k holds, meaning $\mathbf{P}^k$ is true:
\begin{equation*}
    \forall i, \sum_{j=1}^n \mathbf{P}_{ij}^k =1.
\end{equation*}
As $\mathbf{P}^{k+1}=\mathbf{P}^{k}\mathbf{P}$, so we have
\begin{equation*}
    \sum_{j=1}^n\mathbf{P}^{k+1}_{ij}=\sum_{j=1}^n\sum_{k=1}^n\mathbf{P}^{k}_{ik}\mathbf{P}_{kj} = \sum_{k=1}^n\sum_{j=1}^n\mathbf{P}^{k}_{ik}\mathbf{P}_{kj} = \sum_{k=1}^n\mathbf{P}^{k}_{ik}\left(\sum_{j=1}^n\mathbf{P}_{kj}\right) = \sum_{k=1}^n\mathbf{P}^{k}_{ik} = 1,
\end{equation*}
which finishes the proof.
\end{proof}

Lemma~\ref{transition matrix} describes the row summation of $\widetilde{\bm{\mathcal{A}}}^{k}$ is 1. Now we can obtain the row summation for $\widetilde{\bm{\mathcal{A}}}_{S}$.

Then for any $i$, we have
\begin{equation}
\begin{split}
    \sum_{j=1}^{n}\left[\widetilde{\bm{\mathcal{A}}}_{S}\right]_{ij} 
    &=\frac{1}{\lambda+1}\sum_{s=0}^{S} \left(\frac{\lambda}{\lambda+1}\left[\widetilde{\bm{\mathcal{A}}}\right]_{ij}\right)^{s}\\
    &=\frac{1}{\lambda+1}\sum_{s=0}^{S}\left(\frac{\lambda}{\lambda+1}\right)^{s} \\
    &=1-\left(\frac{\lambda}{\lambda+1}\right)^{S+1}.
\end{split}
\end{equation}

\section{Proof of Lemma 1}
\label{appendix:upper bound with hoeffding}
We provide the details of proof of Lemma 1. We first introduce the General Hoeffding Inequality~\citep{hoeffding1994probability}, which is essential for bounding $\left\|\widetilde{\bm{\mathcal{A}}}_S\bm{\eta}\right\|_{F}^{2}$.

\begin{lemma}
\label{hoeffding}
\rm{\textbf{(General Hoeffding Inequality~\citep{hoeffding1994probability})}} \emph{Suppose that the variables $X_{1}, \cdots, X_{n}$ are independent, and $X_i$ has mean $\mu_{i}$ and sub-Gaussian parameter $\sigma_{i}$. Then for all $t\geq0$, we have}
\begin{equation}
 \mathbb{P}\left[\sum_{i=1}^{n}\left(X_{i}-\mu_{i}\right) \geq t\right] \leq \exp \left\{-\frac{t^{2}}{2 \sum_{i=1}^{n} \sigma_{i}^{2}}\right\}.
\end{equation}
\end{lemma}

Now let's prove Lemma~\ref{upper bound of aggregated noised matirx}.
\begin{proof}[Proof of Lemma~\ref{upper bound of aggregated noised matirx}.]
For any entry $\left[\widetilde{\bm{\mathcal{A}}}_S\bm{\eta}\right]_{ij}=\sum_{p=1}^{n}\left(\widetilde{\bm{\mathcal{A}}}_S\right)_{ip}\bm{\eta}_{pj}$, where $\bm{\eta}_{pj}$ is a sub-Gaussian variable with parameter $\sigma^{2}$. By the General Hoeffding inequality~\ref{hoeffding}, we have
\begin{equation}
\mathbb{P}\left(\left|\left[ \frac{1}{\lambda+1}\sum_{s=0}^{S}\left(\frac{\lambda}{\lambda+1}\widetilde{\bm{\mathcal{A}}}_{S}\right)^{s}\bm{\eta}\right]_{ij}\right|\geq t\right) \leq 2\exp \left\{-\frac{nt^{2}}{2\tau\left(1-\left(\frac{\lambda}{\lambda+1}\right)^{S+1}\right)^2\sigma^2}\right\}.
\end{equation}
where $\tau = \max_i \tau_i$ and $ \tau_i = {n\sum_{j=1}^{n}\left[\widetilde{\bm{\mathcal{A}}}_{S}\right]_{ij}^2}\Bigg/{\left(1-\left(\frac{\lambda}{\lambda+1}\right)^{S+1}\right)^2}$. 

Applying union bound~\citep{vershynin2010introduction} to all possible pairs of $i \in [n]$, $j \in [n]$, we get
\begin{equation}
 \mathbb{P}\left(\left\|\widetilde{\bm{\mathcal{A}}}_S\bm{\eta}\right\|_{\infty, \infty}\geq t\right) \leq \sum_{i,j}\mathbb{P}\left(\left[\widetilde{\bm{\mathcal{A}}}_S\bm{\eta}\right]_{ij}\geq t\right) \leq 2n^2\exp \left\{-\frac{nt^{2}}{2\tau\left(1-\left(\frac{\lambda}{\lambda+1}\right)^{S+1}\right)^2\sigma^2}\right\}.  
\end{equation}
Applying union bound again, we have
\begin{equation}
   \mathbb{P}\left(\left\|\widetilde{\bm{\mathcal{A}}}_S\bm{\eta}\right\|_{F}^{2}\geq t\right) \leq \sum_{i,j}\mathbb{P}\left(\left\|\widetilde{\bm{\mathcal{A}}}_S\bm{\eta}\right\|_{\infty, \infty}\geq \sqrt{t}\right)\leq2n^4\exp \left\{-\frac{nt}{2\tau\left(1-\left(\frac{\lambda}{\lambda+1}\right)^{S+1}\right)^2\sigma^2}\right\}.
\end{equation}
Choose $t=2\tau\left(1-\left(\frac{\lambda}{\lambda+1}\right)^{S+1}\right)^2\left(4\log n+\log{2d}\right)/n$ and with probability $1-1/d$, we have 
\begin{equation}
    \left\|\widetilde{\bm{\mathcal{A}}}_S\bm{\eta}\right\|_{F}^{2}\leq \frac{2\tau\left(1-\left(\frac{\lambda}{\lambda+1}\right)^{S+1}\right)^2\sigma^2\left(4\log n+\log{2d}\right)}{n},
\end{equation}
which finishes the proof.
\end{proof}

\section{Proof of the Main Theorem~\ref{theorem:main}}
\label{appendix:upper bound with optimization}
We provide the details of proof of main theorem~\ref{theorem:main}.\\
\textbf{[Restatement of Theorem~\ref{theorem:main}] }
\emph{Under Assumptions 1,2,3,4, let $\mathbf{W}_{f}^{(k)}$ denote the $k$-th step gradient descent solution for $\min_{\mathbf{W}} f(\mathbf{W})$ with step size $\alpha \leq 1 / L$, with probability $1-1/d$ we have}
\begin{equation}
\begin{split}
    g\left(\mathbf{W}_{f}^{(k)}\right)-g\left(\mathbf{W}_{g}^{*}\right)
    &\leq \mathcal{O}\left(\frac{1}{2k\alpha}\right) + \mathcal{O}\left(\frac{\tau \log n}{n}\right),
\end{split}
\end{equation}
where $\mathbf{W}_{g}^{*} = \arg\min_{\mathbf{W}} g(\mathbf{W})$ is the optimal solution of the clean loss function $g(\mathbf{W})$, $\tau$ is the high-order graph connectivity factor, and $n$ is the number of nodes of a graph.

\begin{proof}
By the definition of $L$-smooth, we can obtain the following inequality:
\begin{equation}
    f(\mathbf{W}_{f}^{'}) \leq f(\mathbf{W}_{f})+\langle\nabla f(\mathbf{W}_{f}), \mathbf{W}_{f}^{'}-\mathbf{W}_{f}\rangle+\frac{1}{2} L\|\mathbf{W}_{f}^{'}-\mathbf{W}_{f}\|_{F}^{2}.
\end{equation}
Let's use the gradient descent algorithm with $\mathbf{W}_{f}^{'} = \mathbf{W}_{f}^{+} = \mathbf{W}_{f} - \alpha \nabla f(\mathbf{W}_{f})$. We then get:
\begin{equation}
\label{part 1}
    \begin{split}
        f\left(\mathbf{W}_{f}^{+}\right) &\leq f(\mathbf{W}_{f})+\langle \nabla f(\mathbf{W}_{f}), \mathbf{W}_{f}^{+}-\mathbf{W}_{f}\rangle+\frac{1}{2} L\left\|\mathbf{W}_{f}^{+}-\mathbf{W}_{f}\right\|_{F}^{2}\\
        &=f(\mathbf{W}_{f})+\langle \nabla f(\mathbf{W}_{f}), \mathbf{W}_{f}-\alpha \nabla f(\mathbf{W}_{f})-\mathbf{W}_{f}\rangle+\frac{1}{2} L\|\mathbf{W}_{f}-\alpha \nabla f(\mathbf{W}_{f})-\mathbf{W}_{f}\|_{F}^{2}\\
        &=f(\mathbf{W}_{f})-\langle \nabla f(\mathbf{W}_{f}), \alpha \nabla f(\mathbf{W}_{f}\rangle+\frac{1}{2} L\|\alpha \nabla f(\mathbf{W}_{f})\|_{F}^{2}\\
        &=f(\mathbf{W}_{f})-\alpha\|\nabla f(\mathbf{W}_{f})\|_{F}^{2}+\frac{1}{2} L \alpha^{2}\|\nabla f(\mathbf{W}_{f})\|_{F}^{2}\\
        &=f(\mathbf{W}_{f})-\left(1-\frac{1}{2} L \alpha\right) \alpha\|\nabla f(\mathbf{W}_{f})\|_{F}^{2}.
    \end{split}
\end{equation}
With the fixed step size $\alpha\leq 1/L$, we know that $-(1-\frac{1}{2}L\alpha)=\frac{1}{2}L\alpha-1\leq\frac{1}{2}L(1/L)-1=\frac{1}{2}-1=-\frac{1}{2}$. Plugging this into Eq. (\ref{part 1}), we have the following inequality:
\begin{equation}
\label{part 2}
    f\left(\mathbf{W}_{f}^{+}\right) \leq f(\mathbf{W}_{f})-\frac{1}{2} \alpha\|\nabla f(\mathbf{W}_{f})\|_{F}^{2}.
\end{equation}
If we choose $t$ to be small enough such that $t\leq 1/L$, this inequality implies that the loss function value strictly decreases under each iteration of gradient descent since $\|\nabla f(\mathbf{W}_{f})\|$ is positive unless $\nabla f(\mathbf{W}_{f})=0$ e.g. $\mathbf{W}_{f} = \mathbf{W}_{f}^{*}$, where $\mathbf{W}_{f}$ reaches $\mathbf{W}_{f}^{*}$.

Now, let's bound the loss function value $f(\mathbf{W}_{f}^{+})$. Since $f$ is convex, we can write
\begin{equation}
\label{strctly decrease}
f(\mathbf{W}_{f}) \leq f\left(\mathbf{W}_{g}^{*}\right)+\langle \nabla f(\mathbf{W}_{f}), \mathbf{W}_{f}-\mathbf{W}_{g}^{*}\rangle.
\end{equation}
Introducing this inequality into Eq. (\ref{part 2}),
we can obtain the following:
\begin{equation}
\label{part 3}
\begin{split}
f\left(\mathbf{W}_{f}^{+}\right) - f\left(\mathbf{W}_{g}^{*}\right) &\leq \langle \nabla f(\mathbf{W}_{f}), \mathbf{W}_{f}-\mathbf{W}_{g}^{*}\rangle-\frac{\alpha}{2}\|\nabla f(\mathbf{W}_{f})\|_{F}^{2} \\
&\leq \frac{1}{2\alpha}\left(2\alpha \langle \nabla f(\mathbf{W}_{f}), \mathbf{W}_{f}-\mathbf{W}_{g}^{*}\rangle-\alpha^{2}\|\nabla f(\mathbf{W}_{f})\|_{F}^{2}\right) \\
&\leq \frac{1}{2\alpha}\left(2\alpha\langle \nabla f(\mathbf{W}_{f}), \mathbf{W}_{f}-\mathbf{W}_{g}^{*}\rangle-\alpha^{2}\|\nabla f(\mathbf{W}_{f})\|_{F}^{2}-\left\|\mathbf{W}_{f}-\mathbf{W}_{g}^{*}\right\|_{F}^{2}\right)\\
&\quad+\frac{1}{2\alpha}\left\|\mathbf{W}_{f}-\mathbf{W}_{g}^{*}\right\|_{F}^{2} \\
& \leq \frac{1}{2\alpha}\left(\left\|\mathbf{W}_{f}-\mathbf{W}_{g}^{*}\right\|_{F}^{2}-\left\|\mathbf{W}_{f}-\alpha \nabla f(\mathbf{W}_{f})-\mathbf{W}_{g}^{*}\right\|_{F}^{2}\right).
\end{split}
\end{equation}
Notice that by the definition of gradient descent update, we have $\mathbf{W}_{f}^{+}=\mathbf{W}_{f}-\alpha \nabla f(\mathbf{W}_{f})$. Plugging this into the final inequality of Eq. (\ref{part 3}), we can get:
\begin{equation}
    f\left(\mathbf{W}_{f}^{+}\right)-f\left(\mathbf{W}_{g}^{*}\right) \leq \frac{1}{2\alpha}\left(\left\|\mathbf{W}_{f}-\mathbf{W}_{g}^{*}\right\|_{F}^{2}-\left\|\mathbf{W}_{f}^{+}-\mathbf{W}_{g}^{*}\right\|_{F}^{2}\right).
\end{equation}
This inequality holds for $\mathbf{W}_{f}^{+}$ on every iteration of gradient descent. Summing over iterations, we get:
\begin{equation}
\begin{split}
\sum_{i=1}^{k} \left[f\left(\mathbf{W}_{f}^{(i)}\right)-f\left(\mathbf{W}_{g}^{*}\right)\right]& \leq \sum_{i=1}^{k} \frac{1}{2 \alpha}\left(\left\|\mathbf{W}_{f}^{(i-1)}-\mathbf{W}_{g}^{*}\right\|_{F}^{2}-\left\|\mathbf{W}_{f}^{(i)}-\mathbf{W}_{g}^{*}\right\|_{F}^{2}\right) \\
&=\frac{1}{2\alpha}\left(\left\|\mathbf{W}_{f}^{(0)}-\mathbf{W}_{g}^{*}\right\|_{F}^{2}-\left\|\mathbf{W}_{f}^{(k)}-\mathbf{W}_{g}^{*}\right\|_{F}^{2}\right) \\
& \leq \frac{1}{2\alpha}\left(\left\|\mathbf{W}_{f}^{(0)}-\mathbf{W}_{g}^{*}\right\|_{F}^{2}\right).
\end{split}
\end{equation}
With the inequality of Eq. (\ref{strctly decrease}), we know that $f(\mathbf{W}_{f})$ strictly decreases over each iteration. So we have following:
\begin{equation}
\begin{split}
f\left(\mathbf{W}_{f}^{(k)}\right)-f\left(\mathbf{W}_{g}^{*}\right) &\leq \frac{1}{k} \left[\sum_{i=1}^{k} f\left(\mathbf{W}_{f}^{(i)}\right)-f\left(\mathbf{W}_{g}^{*}\right)\right]\\
& \leq \frac{1}{2k\alpha}\left(\left\|\mathbf{W}_{f}^{(0)}-\mathbf{W}_{g}^{*}\right\|_{F}^{2}\right)
\end{split}
\end{equation}
Equivalently, we have the inequality for the loss function $g(\mathbf{W}_{f})$:
\begin{equation}
\begin{split}
    g\left(\mathbf{W}_{f}^{(k)}\right)-g\left(\mathbf{W}_{g}^{*}\right)&=f\left(\mathbf{W}_{f}^{(k)}\right)-f\left(\mathbf{W}_{g}^{*}\right)\\
    &\quad +2\langle\widetilde{\bm{\mathcal{A}}}_{S}\bm{\eta}\mathbf{W}_{g}^{*}, \widetilde{\bm{\mathcal{A}}}_{S}\mathbf{X}^{*}\mathbf{W}_{g}^{*}-\mathbf{Y}\rangle+\langle\widetilde{\bm{\mathcal{A}}}_{S}\bm{\eta}\mathbf{W}_{g}^{*}, \widetilde{\bm{\mathcal{A}}}_{S}\bm{\eta}\mathbf{W}_{g}^{*}\rangle\\
    &\quad -2\langle\widetilde{\bm{\mathcal{A}}}_{S}\bm{\eta}\mathbf{W}_{f}^{(k)}, \widetilde{\bm{\mathcal{A}}}_{S}\mathbf{X}^{*}\mathbf{W}_{f}^{(k)}-\mathbf{Y}\rangle-\langle\widetilde{\bm{\mathcal{A}}}_{S}\bm{\eta}\mathbf{W}_{f}^{(k)}, \widetilde{\bm{\mathcal{A}}}_{S}\bm{\eta}\mathbf{W}_{f}^{(k)}\rangle\\
    &\leq \frac{1}{2k\alpha}\left(\left\|\mathbf{W}_{f}^{(0)}-\mathbf{W}_{g}^{*}\right\|_{F}^{2}\right)\\
    &\quad +\left\|\widetilde{\bm{\mathcal{A}}}_{S}\bm{\eta}\right\|_{F}^{2}\left\|\mathbf{W}_{g}^{*}\right\|_{F}^2\left(2\left\|\widetilde{\bm{\mathcal{A}}}_{S}\mathbf{X}^{*}\mathbf{W}_{g}^{*}-\mathbf{Y}\right\|_{F}^{2}+\left\| \widetilde{\bm{\mathcal{A}}}_{S}\bm{\eta}\right\|_{F}^{2}\left\|\mathbf{W}_{g}^{*}\right\|_{F}^{2}\right)\\
    &\quad +\left\|\widetilde{\bm{\mathcal{A}}}_{S}\bm{\eta}\right\|_{F}^{2}\left\|\mathbf{W}_{f}^{(k)}\right\|_{F}^2\left(2\left\|\widetilde{\bm{\mathcal{A}}}_{S}\mathbf{X}\mathbf{W}_{f}^{(k)}-\mathbf{Y}\right\|_{F}^{2}+\left\| \widetilde{\bm{\mathcal{A}}}_{S}\bm{\eta}\right\|_{F}^{2}\left\|\mathbf{W}_{f}^{(k)}\right\|_{F}^{2}\right)\\
    &\leq \mathcal{O}\left(\frac{1}{2k\alpha}\right) + \mathcal{O}\left(\frac{\tau \log n}{n}\right),
\end{split}
\end{equation}
which finishes the proof.
\end{proof}

\section{More Details on Equation~(\ref{graph signal denoising}).}
We provide more details on how to obtain Equation~(\ref{graph signal denoising}).

Note that if we set $\widetilde{\mathbf{L}}=\mathbf{I}-\widetilde{\mathbf{D}}^{-\frac{1}{2}}\widetilde{\mathbf{A}} \widetilde{\mathbf{D}}^{-\frac{1}{2}}$, we have $\operatorname{tr}\left(\mathbf{F}^{\top} \widetilde{\mathbf{L}} \mathbf{F}\right)=\operatorname{tr}\left(\mathbf{F}^{\top} (\mathbf{I}-\widetilde{\mathbf{D}}^{-\frac{1}{2}}\widetilde{\mathbf{A}} \widetilde{\mathbf{D}}^{-\frac{1}{2}}) \mathbf{F}\right)=\operatorname{tr}\left(\mathbf{F}^{\top}\mathbf{F}\right)-\operatorname{tr}\left(\mathbf{F}^{\top} \widetilde{\mathbf{D}}^{-\frac{1}{2}}\widetilde{\mathbf{A}} \widetilde{\mathbf{D}}^{-\frac{1}{2}} \mathbf{F}\right)=\operatorname{tr}\left(\mathbf{F}\mathbf{F}^{\top}\right)-\operatorname{tr}\left( \widetilde{\mathbf{D}}^{-\frac{1}{2}}\widetilde{\mathbf{A}} \widetilde{\mathbf{D}}^{-\frac{1}{2}} \mathbf{F}\mathbf{F}^{\top}\right)$. 
On the other hand, if we set $\widetilde{\mathbf{L}}=\mathbf{I}-\widetilde{\mathbf{D}}^{-1}\widetilde{\mathbf{A}}$, we have $\operatorname{tr}\left(\mathbf{F}^{\top} \widetilde{\mathbf{L}} \mathbf{F}\right)=\operatorname{tr}\left(\mathbf{F}^{\top} (\mathbf{I}-\widetilde{\mathbf{D}}^{-1}\widetilde{\mathbf{A}}) \mathbf{F}\right)=\operatorname{tr}\left(\mathbf{F}^{\top}\mathbf{F}\right)-\operatorname{tr}\left(\mathbf{F}^{\top} \widetilde{\mathbf{D}}^{-1}\widetilde{\mathbf{A}} \mathbf{F}\right)=\operatorname{tr}\left(\mathbf{F}\mathbf{F}^{\top}\right)-\operatorname{tr}\left( \widetilde{\mathbf{D}}^{-1}\widetilde{\mathbf{A}}\mathbf{F}\mathbf{F}^{\top}\right)$. We denote
$\mathbf{F}=\left[\begin{array}{c}
\mathbf{F}_{1} \\
\vdots \\
\mathbf{F}_{n} \\
\end{array}\right]$ and $\mathbf{F}^{\top}=\left[\mathbf{F}_{1}^{\top} \cdots \mathbf{F}_{n}^{\top}\right]$, where $\mathbf{F}_i=\left[\mathbf{F}_{i1} \cdots \mathbf{F}_{id}\right]$, then we have $\operatorname{tr}\left(\mathbf{F}\mathbf{F}^{\top}\right) = \sum_{i=1}^n \mathbf{F}_{i}\mathbf{F}^{\top}_{i}$. \\
When $\widetilde{\mathbf{L}}=\mathbf{I}-\widetilde{\mathbf{D}}^{-\frac{1}{2}}\widetilde{\mathbf{A}} \widetilde{\mathbf{D}}^{-\frac{1}{2}}$, we have
\begin{equation*}
\begin{split}
&\quad\operatorname{tr}\left( \widetilde{\mathbf{D}}^{-\frac{1}{2}}\widetilde{\mathbf{A}} \widetilde{\mathbf{D}}^{-\frac{1}{2}} \mathbf{F}\mathbf{F}^{\top}\right)\\
&=\operatorname{tr}\left(\left[\begin{array}{cccc}
\frac{\mathbf{A}_{11}}{\sqrt{d_{1}+1}\sqrt{d_{1}+1}} & \frac{\mathbf{A}_{12}}{\sqrt{d_{1}+1}\sqrt{d_{2}+1}}  & \cdots & \frac{\mathbf{A}_{1n}}{\sqrt{d_{1}+1}\sqrt{d_{n}+1}} \\
\frac{\mathbf{A}_{21}}{\sqrt{d_{2}+1}\sqrt{d_{1}+1}} & \frac{\mathbf{A}_{22}}{\sqrt{d_{2}+1}\sqrt{d_{2}+1}}  & \cdots & \frac{\mathbf{A}_{2n}}{\sqrt{d_{2}+1}\sqrt{d_{n}+1}} \\
\vdots  & \ddots  & \ddots & \vdots \\
\frac{\mathbf{A}_{n1}}{\sqrt{d_{n}+1}\sqrt{d_{1}+1}} & \frac{\mathbf{A}_{n2}}{\sqrt{d_{n}+1}\sqrt{d_{2}+1}} & \cdots & \frac{\mathbf{A}_{nn}}{\sqrt{d_{n}+1}\sqrt{d_{n}+1}}
\end{array}\right] \left[\begin{array}{cccc}
\mathbf{F}_{1}\mathbf{F}_{1}^{\top} & \mathbf{F}_{1}\mathbf{F}_{2}^{\top}  & \cdots & \mathbf{F}_{1}\mathbf{F}_{n}^{\top} \\
\mathbf{F}_{2}\mathbf{F}_{1}^{\top} & \mathbf{F}_{2}\mathbf{F}_{2}^{\top} & \cdots & \mathbf{F}_{2}\mathbf{F}_{n}^{\top} \\
\vdots  & \ddots  & \ddots & \vdots \\
\mathbf{F}_{n}\mathbf{F}_{1}^{\top} & \mathbf{F}_{n}\mathbf{F}_{2}^{\top} & \cdots & \mathbf{F}_{n}\mathbf{F}_{n}^{\top}
\end{array}\right]\right)\\
&=\sum_{i=1}^{n}\sum_{j=1}^{n}\frac{\mathbf{A}_{ij}}{\sqrt{d_{i}+1}\sqrt{d_{j}+1}}\mathbf{F}_{j}\mathbf{F}_{i}^{\top}.
\end{split}
\end{equation*}
On the other hand, when $\widetilde{\mathbf{L}}=\mathbf{I}-\widetilde{\mathbf{D}}^{-1}\widetilde{\mathbf{A}}$, we have
\begin{equation*}
\begin{split}
&\quad\operatorname{tr}\left( \widetilde{\mathbf{D}}^{-1}\widetilde{\mathbf{A}} \mathbf{F}\mathbf{F}^{\top}\right)\\
&=\operatorname{tr}\left(\left[\begin{array}{cccc}
\frac{\mathbf{A}_{11}}{d_{1}+1} & \frac{\mathbf{A}_{12}}{d_{1}+1}  & \cdots & \frac{\mathbf{A}_{1n}}{d_{1}+1} \\
\frac{\mathbf{A}_{21}}{d_{2}+1} & \frac{\mathbf{A}_{22}}{d_{2}+1} & \cdots & \frac{\mathbf{A}_{2n}}{d_{2}+1}\\
\vdots  & \ddots  & \ddots & \vdots \\
\frac{\mathbf{A}_{n1}}{d_{n}+1} & \frac{\mathbf{A}_{n2}}{d_{n}+1} & \cdots & \frac{\mathbf{A}_{nn}}{d_{n}+1}
\end{array}\right] \left[\begin{array}{cccc}
\mathbf{F}_{1}\mathbf{F}_{1}^{\top} & \mathbf{F}_{1}\mathbf{F}_{2}^{\top}  & \cdots & \mathbf{F}_{1}\mathbf{F}_{n}^{\top} \\
\mathbf{F}_{2}\mathbf{F}_{1}^{\top} & \mathbf{F}_{2}\mathbf{F}_{2}^{\top} & \cdots & \mathbf{F}_{2}\mathbf{F}_{n}^{\top} \\
\vdots  & \ddots  & \ddots & \vdots \\
\mathbf{F}_{n}\mathbf{F}_{1}^{\top} & \mathbf{F}_{n}\mathbf{F}_{2}^{\top} & \cdots & \mathbf{F}_{n}\mathbf{F}_{n}^{\top}
\end{array}\right]\right)\\
&=\sum_{i=1}^{n}\sum_{j=1}^{n}\frac{\mathbf{A}_{ij}}{d_{i}+1}\mathbf{F}_{j}\mathbf{F}_{i}^{\top}.
\end{split}
\end{equation*}
So when $\widetilde{\mathbf{L}}=\mathbf{I}-\widetilde{\mathbf{D}}^{-\frac{1}{2}}\widetilde{\mathbf{A}} \widetilde{\mathbf{D}}^{-\frac{1}{2}}$, we have 
\begin{equation*}
    \begin{split}
        &\quad\operatorname{tr}\left(\mathbf{F}^{\top} \widetilde{\mathbf{L}} \mathbf{F}\right)\quad \left(\widetilde{\mathbf{L}}=\mathbf{I}-\widetilde{\mathbf{D}}^{-\frac{1}{2}}\widetilde{\mathbf{A}} \widetilde{\mathbf{D}}^{-\frac{1}{2}}\right)\\
        &= \operatorname{tr}\left(\mathbf{F}^{\top} (\mathbf{I}-\widetilde{\mathbf{D}}^{-\frac{1}{2}}\widetilde{\mathbf{A}} \widetilde{\mathbf{D}}^{-\frac{1}{2}}) \mathbf{F}\right) \\
        &=\operatorname{tr}\left(\mathbf{F}\mathbf{F}^{\top}\right)-\operatorname{tr}\left( \widetilde{\mathbf{D}}^{-\frac{1}{2}}\widetilde{\mathbf{A}} \widetilde{\mathbf{D}}^{-\frac{1}{2}} \mathbf{F}\mathbf{F}^{\top}\right) \\
        &= \sum_{i=1}^n \mathbf{F}_{i}\mathbf{F}^{\top}_{i} -\sum_{i=1}^{n}\sum_{j=1}^{n}\frac{\mathbf{A}_{ij}}{\sqrt{d_{i}+1}\sqrt{d_{j}+1}}\mathbf{F}_{j}\mathbf{F}_{i}^{\top} \\
        &= \frac{1}{2}\sum_{i=1}^n \mathbf{F}_{i}\mathbf{F}^{\top}_{i} + \frac{1}{2}\sum_{j=1}^n \mathbf{F}_{j}\mathbf{F}^{\top}_{j} -\sum_{i=1}^{n}\sum_{j=1}^{n}\frac{\mathbf{A}_{ij}}{\sqrt{d_{i}+1}\sqrt{d_{j}+1}}\mathbf{F}_{j}\mathbf{F}_{i}^{\top} \\
        &=\frac{1}{2}\left(\sum_{i=1}^n \mathbf{F}_{i}\mathbf{F}^{\top}_{i} + \sum_{j=1}^n \mathbf{F}_{j}\mathbf{F}^{\top}_{j} -2\sum_{i=1}^{n}\sum_{j=1}^{n}\frac{\mathbf{A}_{ij}}{\sqrt{d_{i}+1}\sqrt{d_{j}+1}}\mathbf{F}_{j}\mathbf{F}_{i}^{\top}\right) \\
        &= \frac{1}{2}\left( \sum_{i=1}^n \sum_{j=1}^n \frac{\mathbf{A}_{ij}\mathbf{F}_{i}\mathbf{F}^{\top}_{i}}{d_i+1} +  \sum_{i=1}^n \sum_{j=1}^n \frac{\mathbf{A}_{ij}\mathbf{F}_{j}\mathbf{F}^{\top}_{j}}{d_j+1} -2\sum_{i=1}^{n}\sum_{j=1}^{n}\frac{\mathbf{A}_{ij}}{\sqrt{d_{i}+1}\sqrt{d_{j}+1}}\mathbf{F}_{j}\mathbf{F}_{i}^{\top}\right) \text{undirected graph}\\
        &=\frac{1}{2}\left(\sum_{i=1}^{n}\sum_{j=1}^{n}\left(\frac{\mathbf{A}_{ij}\mathbf{F}_{i}\mathbf{F}^{\top}_{i}}{d_i+1} +\frac{\mathbf{A}_{ij}\mathbf{F}_{j}\mathbf{F}^{\top}_{j}}{d_j+1}-\frac{\mathbf{A}_{ij}}{\sqrt{d_{i}+1}\sqrt{d_{j}+1}}\mathbf{F}_{j}\mathbf{F}_{i}^{\top}-\frac{\mathbf{A}_{ij}}{\sqrt{d_{i}+1}\sqrt{d_{j}+1}}\mathbf{F}_{i}\mathbf{F}_{j}^{\top}\right)\right)\\
        &=\frac{1}{2}\left(\sum_{i=1}^{n}\sum_{j=1}^{n}\mathbf{A}_{ij}\left(\frac{\mathbf{F}_{i}\mathbf{F}^{\top}_{i}}{d_i+1} +\frac{\mathbf{F}_{j}\mathbf{F}^{\top}_{j}}{d_j+1}-\frac{\mathbf{F}_{j}\mathbf{F}_{i}^{\top}}{\sqrt{d_{i}+1}\sqrt{d_{j}+1}}-\frac{\mathbf{F}_{i}\mathbf{F}_{j}^{\top}}{\sqrt{d_{i}+1}\sqrt{d_{j}+1}}\right)\right)\\
        &=\frac{1}{2}\left(\sum_{i=1}^{n}\sum_{j=1}^{n}\mathbf{A}_{ij}\left(\frac{\mathbf{F}_i}{\sqrt{d_i+1}}-\frac{\mathbf{F}_j}{\sqrt{d_j+1}}\right)\left(\frac{\mathbf{F}_i^{\top}}{\sqrt{d_i+1}}-\frac{\mathbf{F}_j^{\top}}{\sqrt{d_j+1}}\right)\right)\\
        &=\frac{1}{2}\left(\sum_{i=1}^{n}\sum_{j=1}^{n}\mathbf{A}_{ij}\left\|\frac{\mathbf{F}_{i}}{\sqrt{d_{i}+1}}-\frac{\mathbf{F}_{j}}{\sqrt{d_{j}+1}}\right\|_{2}^{2}\right) = \sum_{(i, j) \in \mathcal{E}} \mathbf{A}_{i j}\left\|\frac{\mathbf{F}_{i}}{\sqrt{d_{i}+1}}-\frac{\mathbf{F}_{j}}{\sqrt{d_{j}+1}}\right\|_{2}^{2}.
    \end{split}
\end{equation*}
On the other hand, when $\widetilde{\mathbf{L}}=\mathbf{I}-\widetilde{\mathbf{D}}^{-1}\widetilde{\mathbf{A}}$, we have
\begin{equation*}
    \begin{split}
        &\operatorname{tr}\left(\mathbf{F}^{\top} \widetilde{\mathbf{L}} \mathbf{F}\right)\quad \left(\widetilde{\mathbf{L}}=\mathbf{I}-\widetilde{\mathbf{D}}^{-1}\widetilde{\mathbf{A}} \right)\\
        &= \operatorname{tr}\left(\mathbf{F}^{\top} (\mathbf{I}-\widetilde{\mathbf{D}}^{-1}\widetilde{\mathbf{A}}) \mathbf{F}\right) \\
        &= \sum_{i=1}^n \mathbf{F}_{i}\mathbf{F}^{\top}_{i} -\sum_{i=1}^{n}\sum_{j=1}^{n}\frac{\mathbf{A}_{ij}}{d_{i}+1}\mathbf{F}_{j}\mathbf{F}_{i}^{\top} \\
        &= \frac{1}{2}\sum_{i=1}^n \mathbf{F}_{i}\mathbf{F}^{\top}_{i} + \frac{1}{2}\sum_{j=1}^n \mathbf{F}_{j}\mathbf{F}^{\top}_{j} -\sum_{i=1}^{n}\sum_{j=1}^{n}\frac{\mathbf{A}_{ij}}{d_{i}+1}\mathbf{F}_{j}\mathbf{F}_{i}^{\top} \\
        &= \frac{1}{2}\left( \sum_{i=1}^n \sum_{j=1}^n \frac{\mathbf{A}_{ij}\mathbf{F}_{i}\mathbf{F}^{\top}_{i}}{d_i+1} +  \sum_{i=1}^n \sum_{j=1}^n \frac{\mathbf{A}_{ij}\mathbf{F}_{j}\mathbf{F}^{\top}_{j}}{d_i+1} -2\sum_{i=1}^{n}\sum_{j=1}^{n}\frac{\mathbf{A}_{ij}}{\sqrt{d_{i}+1}\sqrt{d_{i}+1}}\mathbf{F}_{j}\mathbf{F}_{i}^{\top}\right) \text{undirected graph}\\
        &=\frac{1}{2}\left(\sum_{i=1}^{n}\sum_{j=1}^{n}\mathbf{A}_{ij}\left(\frac{\mathbf{F}_i}{\sqrt{d_i+1}}-\frac{\mathbf{F}_j}{\sqrt{d_i+1}}\right)\left(\frac{\mathbf{F}_i^{\top}}{\sqrt{d_i+1}}-\frac{\mathbf{F}_j^{\top}}{\sqrt{d_i+1}}\right)\right)\\
        &=\frac{1}{2}\left(\sum_{i=1}^{n}\sum_{j=1}^{n}\mathbf{A}_{ij}\left\|\frac{\mathbf{F}_{i}}{\sqrt{d_{i}+1}}-\frac{\mathbf{F}_{j}}{\sqrt{d_{i}+1}}\right\|_{2}^{2}\right) = \sum_{(i, j) \in \mathcal{E}} \mathbf{A}_{i j}\left\|\frac{\mathbf{F}_{i}}{\sqrt{d_{i}+1}}-\frac{\mathbf{F}_{j}}{\sqrt{d_{i}+1}}\right\|_{2}^{2}.
    \end{split}
\end{equation*}

\section{Datasets Details}
Cora, Citeseer, and Pubmed are standard citation network benchmark datasets~\citep{sen2008collective}. Coauthor-CS and Coauthor-Phy are extracted from Microsoft Academic Graph~\citep{shchur2018pitfalls}. Cornell, Texas, Wisconsin, and Actor are constructed by \citet{pei2020geom}. ogbn-products is a large-scale product, constructed by \citet{hu2020open}. 
\begin{table}[!hbtp]
\caption{Datasets statistics}
\label{tab:stat}
\setlength{\tabcolsep}{2mm}
\begin{center}
\begin{tabular}{lcccc}
\toprule
  \multicolumn{1}{l}{\textbf{Dataset}}  & \# Nodes    & \# Edges   & \# Features  & \# Classes \\
\midrule
Cora  &  2708   & 5429 & 1433 & 7       \\
Citeseer    & 3327 & 4732 & 3703 & 6 \\
Pubmed   & 19717 & 44338 & 500 & 3   \\
Cornell   & 183 & 295 & 1703 & 5   \\
Texas & 183 & 309 & 1703 & 5 \\
Wisconsin & 251 & 499 & 1703 & 5 \\
Actor & 7600 & 33544 & 931 & 5 \\
Coauthor-CS & 18333 & 81894 & 6805 & 15 \\
Coauthor-Phy & 34493 & 247962 & 8415 & 5 \\
ogbn-products & 2449029 & 61859140 & 100 & 42\\
\bottomrule
\end{tabular}
\end{center}
\end{table}

\section{Reproducibility}
\subsection{Implementation Details}
We use Pytorch~\citep{paszke2019pytorch} and PyG~\citep{fey2019fast} to implement \name and R\name. The codes of baselines are implemented referring to the implementation of MLP\footnote{https://github.com/tkipf/pygcn}\footnote{https://github.com/snap-stanford/ogb/blob/master/examples/nodeproppred/products/mlp.py}, GCN\footnote{https://github.com/tkipf/pygcn}\footnote{https://github.com/snap-stanford/ogb/blob/master/examples/nodeproppred/products/gnn.py}, GAT\footnote{https://github.com/pyg-team/pytorch\_geometric/blob/master/examples/gat.py}, GLP\footnote{https://github.com/liqimai/Efficient-SSL}, S$^2$GC\footnote{https://github.com/allenhaozhu/SSGC}, and IRLS\footnote{https://github.com/FFTYYY/TWIRLS}. All the experiments in this work are conducted on a single NVIDIA Tesla A100 with 80GB memory size. The software that we use for experiments are Python 3.6.8, pytorch 1.9.0, pytorch-scatter 2.0.9, pytorch-sparse 0.6.12, pyg 2.0.3, ogb 1.3.4, numpy 1.19.5, torchvision 0.10.0, and CUDA 11.1.

\subsection{Hyperparameter Details}
\label{appendix:hyperparameter}
We provide details about hyparatemeters of \name and R\name in Table~\ref{tab:citation_hyper}, \ref{tab:heterophily_hyper}, \ref{tab:co-author_hyper}, \ref{tab:products_hyper}, and \ref{tab:citation_hyper_flip}.
\begin{table}[t]
\setlength{\tabcolsep}{0.6mm}
\caption{The hyper-parameters for \name and R\name on three citation datasets.}
\label{tab:citation_hyper}
\vskip 0.15in
    \centering
    \begin{tabular}{l|c|c|c|c|c|c|c|c|c|c}
    \toprule
    Model & dataset & runs  & lr & epochs & wight decay & hidden & dropout & $S$ & $\lambda$ & $\epsilon$\\
    \midrule
    \name & Cora & 100 & 0.2 & 100 & 1e-5 & 0 & 0 & 16 & 32 & - \\
    \name & Citeseer & 100 & 0.2 & 100 & 1e-5 & 0 & 0 & 16 & 32 & - \\
    \name & Pubmed & 100 & 0.2 & 100 & 1e-5 & 0 & 0 & 16 & 32 & - \\
    \hline
    R\name & Cora & 100 & 0.2 & 100 & 1e-5 & 0 & 0 & 16 & 32 & 1 \\
    R\name & Citeseer & 100 & 0.2 & 100 & 1e-5 & 0 & 0 & 16 & 32 & 1 \\
    R\name & Pubmed & 100 & 0.2 & 100 & 1e-5 & 0 & 0 & 16 & 32 & 1 \\
    \bottomrule
    \end{tabular}
\end{table}

\begin{table}[t]
\setlength{\tabcolsep}{0.6mm}
\caption{The hyper-parameters for \name and R\name on four heterophily graphs.}
\label{tab:heterophily_hyper}
\vskip 0.15in
    \centering
    \begin{tabular}{l|c|c|c|c|c|c|c|c|c|c|c|c}
    \toprule
    Model & dataset & noise level & runs  & lr & epochs & wight decay & hidden & dropout & $S$ & $\lambda$ & $\epsilon$ & +MLP\\
    \midrule
    \name & Cornell & 0.01 & 10 & 0.2 & 200 & 5e-4 & 16 & 0.5 & 16 & 1 & - & y \\
    \name & Cornell & 1 & 10 & 0.2 & 200 & 5e-4 & 16 & 0.5 & 16 & 1024 & - & y \\
    \name & Texas & 0.01 & 10 & 0.2 & 200 & 5e-4 & 16 & 0.5 & 16 & 1 & - & y \\
    \name & Texas & 1 & 10 & 0.2 & 200 & 5e-4 & 16 & 0.5 & 16 & 1024 & - & y \\
    \name & Wisconsin & 0.01 & 10 & 0.2 & 1000 & 5e-4 & 16 & 0.5 & 2 & 1 & - & y \\
    \name & Wisconsin & 1 & 10 & 0.2 & 1000 & 5e-4 & 16 & 0.5 & 2 & 1024 & - & y \\
    \name & Actor & 0.01 & 10 & 0.2 & 1000 & 5e-4 & 16 & 0.5 & 2 & 1 & - & y \\
    \name & Actor & 1 & 10 & 0.2 & 1000 & 5e-4 & 16 & 0.5 & 2 & 1024 & - & y \\
    \hline
    R\name & Cornell & 0.01 & 10 & 0.2 & 200 & 5e-4 & 16 & 0.5 & 16 & 1 & 1 & y \\
    R\name & Cornell & 1 & 10 & 0.2 & 200 & 5e-4 & 16 & 0.5 & 16 & 1024 & 1 & y \\
    R\name & Texas & 0.01 & 10 & 0.2 & 200 & 5e-4 & 16 & 0.5 & 16 & 1 & 1 & y \\
    R\name & Texas & 1 & 10 & 0.2 & 200 & 5e-4 & 16 & 0.5 & 16 & 1024 & 1 & y \\
    R\name & Wisconsin & 0.01 & 10 & 0.2 & 1000 & 5e-4 & 16 & 0.5 & 2 & 1 & 1e-5 & y \\
    R\name & Wisconsin & 1 & 10 & 0.2 & 1000 & 5e-4 & 16 & 0.5 & 2 & 1024 & 1e-5 & y \\
    R\name & Actor & 0.01 & 10 & 0.2 & 1000 & 5e-4 & 16 & 0.5 & 2 & 1 & 1e-5 & y \\
    R\name & Actor & 1 & 10 & 0.2 & 1000 & 5e-4 & 16 & 0.5 & 2 & 1024 & 1e-5 & y \\
    \bottomrule
    \end{tabular}
\end{table}

\begin{table}[t]
\setlength{\tabcolsep}{0.6mm}
\caption{The hyper-parameters for \name and R\name on two co-author datasets.}
\label{tab:co-author_hyper}
\vskip 0.15in
    \centering
    \begin{tabular}{l|c|c|c|c|c|c|c|c|c|c|c}
    \toprule
    Model & dataset & noise level & runs  & lr & epochs & wight decay & hidden & dropout & $S$ & $\lambda$ & $\epsilon$ \\
    \midrule
    \name & Coauthor-CS & 0.1 & 10 & 0.2 & 1000 & 1e-7 & 0 & 0 & 16 & 1 & -  \\
    \name & Coauthor-CS & 1 & 10 & 0.2 & 1000 & 1e-7 & 0 & 0 & 16 & 128 & -  \\
    \name & Coauthor-Phy & 0.1 & 10 & 0.2 & 200 & 5e-4 & 16 & 0.5 & 16 & 1 & -  \\
    \name & Coauthor-Phy & 1 & 10 & 0.2 & 200 & 5e-4 & 16 & 0.5 & 16 & 1024 & -  \\
    \hline
    R\name & Coauthor-CS & 0.1 & 10 & 0.2 & 1000 & 1e-7 & 0 & 0 & 16 & 1 & 1  \\
    R\name & Coauthor-CS & 1 & 10 & 0.2 & 1000 & 1e-7 & 0 & 0 & 16 & 128 & 1  \\
    R\name & Coauthor-Phy & 0.1 & 10 & 0.2 & 200 & 5e-4 & 16 & 0.5 & 16 & 1 & 1  \\
    R\name & Coauthor-Phy & 1 & 10 & 0.2 & 200 & 5e-4 & 16 & 0.5 & 16 & 1024 & 1  \\
    \bottomrule
    \end{tabular}
\end{table}

\begin{table}[t]
\setlength{\tabcolsep}{0.6mm}
\caption{The hyper-parameters for \name and R\name on ogbn-products dataset.}
\label{tab:products_hyper}
\vskip 0.15in
    \centering
    \begin{tabular}{l|c|c|c|c|c|c|c|c|c|c|c}
    \toprule
    Model  & noise level & runs  & lr & epochs & hidden & dropout & $S$ & $\lambda$ & $\epsilon$ & layers & +MLP \\
    \midrule
    \name & 0.1 & 10 & 0.01 & 300 & 256 & 0.5 & 128 & 32 & -  & 3 & y \\
    \name & 1 & 10 & 0.01 & 300 & 256 & 0.5 & 128 & 256 & -  & 3 & y\\
    \hline
    R\name & 0.1 & 10 & 0.01 & 300 & 256 & 0.5 & 128 & 32 & 1e-2  & 3 & y \\
    R\name & 1 & 10 & 0.01 & 300 & 256 & 0.5 & 128 & 256 & 1e-2  & 3 & y\\
    \bottomrule
    \end{tabular}
\end{table}

\begin{table}[t]
\setlength{\tabcolsep}{0.6mm}
\caption{The hyper-parameters for \name and R\name on three citation datasets of the flipping experiments.}
\label{tab:citation_hyper_flip}
\vskip 0.15in
    \centering
    \begin{tabular}{l|c|c|c|c|c|c|c|c|c|c|c}
    \toprule
    Model & dataset & flip probability & runs  & lr & epochs & wight decay & hidden & dropout & $S$ & $\lambda$ & $\epsilon$\\
    \midrule
    \name & Cora  & 0.1 & 100 & 0.2 & 100 & 1e-5 & 0 & 0 & 32 & 64 & - \\
    \name & Cora  & 0.2 & 100 & 0.2 & 100 & 1e-5 & 0 & 0 & 16 & 32 & - \\
    \name & Cora  & 0.4 & 100 & 0.2 & 100 & 1e-5 & 0 & 0 & 16 & 32 & - \\
    \name & Citeseer  & 0.1 & 100 & 0.2 & 100 & 1e-5 & 0 & 0 & 16 & 32 & - \\
    \name & Citeseer  & 0.2 & 100 & 0.2 & 100 & 1e-5 & 0 & 0 & 16 & 32 & - \\
    \name & Citeseer  & 0.4 & 100 & 0.2 & 100 & 1e-5 & 0 & 0 & 16 & 32 & - \\
    \name & Pubmed  & 0.1 & 100 & 0.2 & 100 & 1e-5 & 0 & 0 & 16 & 32 & - \\
    \name & Pubmed  & 0.2 & 100 & 0.2 & 100 & 1e-5 & 0 & 0 & 16 & 32 & - \\
    \name & Pubmed  & 0.4 & 100 & 0.2 & 100 & 1e-5 & 0 & 0 & 16 & 32 & - \\
    \hline
    R\name & Cora  & 0.1 & 100 & 0.2 & 100 & 1e-5 & 0 & 0 & 32 & 64 & 1e-5 \\
    R\name & Cora  & 0.2 & 100 & 0.2 & 100 & 1e-5 & 0 & 0 & 16 & 32 & 1e-5 \\
    R\name & Cora  & 0.4 & 100 & 0.2 & 100 & 1e-5 & 0 & 0 & 16 & 32 & 1e-1 \\
    R\name & Citeseer  & 0.1 & 100 & 0.2 & 100 & 1e-5 & 0 & 0 & 16 & 32 & 1e-5 \\
    R\name & Citeseer  & 0.2 & 100 & 0.2 & 100 & 1e-5 & 0 & 0 & 16 & 32 & 1e-5 \\
    R\name & Citeseer  & 0.4 & 100 & 0.2 & 100 & 1e-5 & 0 & 0 & 16 & 32 & 1e-5 \\
    R\name & Pubmed  & 0.1 & 100 & 0.2 & 100 & 1e-5 & 0 & 0 & 16 & 32 & 1e-1 \\
    R\name & Pubmed  & 0.2 & 100 & 0.2 & 100 & 1e-5 & 0 & 0 & 16 & 32 & 1e-1 \\
    R\name & Pubmed  & 0.4 & 100 & 0.2 & 100 & 1e-5 & 0 & 0 & 16 & 32 & 1e-1 \\
    \bottomrule
    \end{tabular}
\end{table}

\section{Additional Experiments}

\subsection{Analysis on Row Normalization}
\label{appendix:row norm}

\begin{table}[htbp]
\normalsize
\caption{Summary of results of NGC w/o raw normalization on three datasets in terms of classification accuracy (\%)}
\label{tab:row normalization}
\setlength{\tabcolsep}{2.5mm}
\begin{center}
\begin{tabular}{cccccccccc}
\toprule
\multirow{2}{*}{Noise Level} & \multicolumn{3}{c}{Cora} & \multicolumn{3}{c}{Citeseer} & \multicolumn{3}{c}{Pubmed} \\
\cmidrule(r){2-4} \cmidrule(r){5-7} \cmidrule(r){8-10}
&  1     &  10  &   100
&  1     &  10  &   100
&  1     &  10  &   100 \\
\midrule
w/o RN   &68.3  &59.7  &56.1   &43.5  &40.4  &37.6  &43.1  &38.8  &37.4 \\
w RN  &66.1  &65.5  &66.2  &45.3 &45.1  &44.8  &62.3  &62.7  &62.1 \\
\bottomrule
\end{tabular}
\end{center}
\end{table}

In this section, we analyze the influence of row normalization on denoising performance. The noise level $\xi$ controls the magnitude of the Gaussian noise we add to the feature matrix: $\mathbf{X}+\xi\bm{\eta}$ where $\bm{\eta}$ is sampled from standard i.i.d., Gaussian distribution. For Cora, Citeseer, and Pubmed, we test $\xi \in \{1, 10, 100\}$. From Table~\ref{tab:row normalization}, we can observe that the denoising performance of w/ row normalization is better than w/o row normalization. Since row normalization can shrink the value of elements in $\bm{\eta}$, thus reducing the variance $\sigma$. In other words, row normalization make $\left\|\widetilde{\bm{\mathcal{A}}}_S\bm{\eta}\right\|_{F}^{2}$ converge to zero faster.

\subsection{Analysis on the Depth of \name and R\name}
\label{appendix:depth analysis}
In this section, we analyze the influence of the depth of \name and R\name model on denoising performance by testing the classification accuracy on semi-supervised node classification tasks. We conduct two sets of experiments: with/without noise in feature matrix. For experiment with feature noise, we simple fix the noise level $\xi =1$. In each set of experiments, we evaluate the test accuracy with respect to \name and R\name model depth, which corresponding to the value of $S$ in $\widetilde{\mathcal{A}}_{S}$. From Figure~\ref{fig:depth_ngc} and \ref{fig:depth_rngc}, we can observe that the test accuracy barely changes with depth if the model is trained on the clean features on Cora and Pubmed but changes greatly if the model is trained on the clean feature on Citeseer. In this regard, the over-smoothing issue exists in R\name model on citeseer. However, the denoising performance of shallow R\name is not good as deeper R\name models, especially on the large graph like Pubmed. This suggests that we do need to increase the depth of GNN model to include more higher-order neighbors for better denoising performances. 

 \begin{figure}[t]
    \begin{center}
        \includegraphics[width=0.325\textwidth]{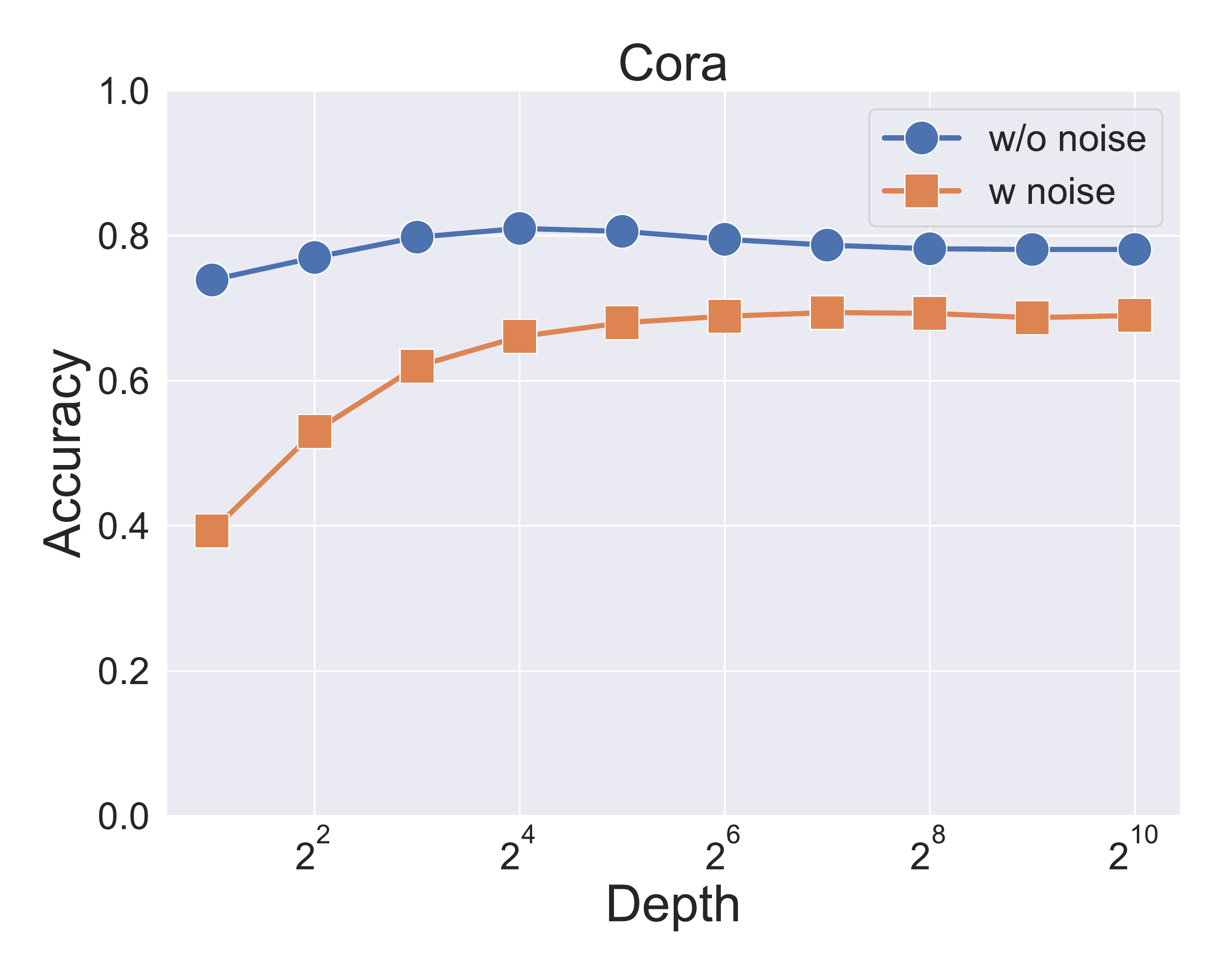}
        \includegraphics[width=0.325\textwidth]{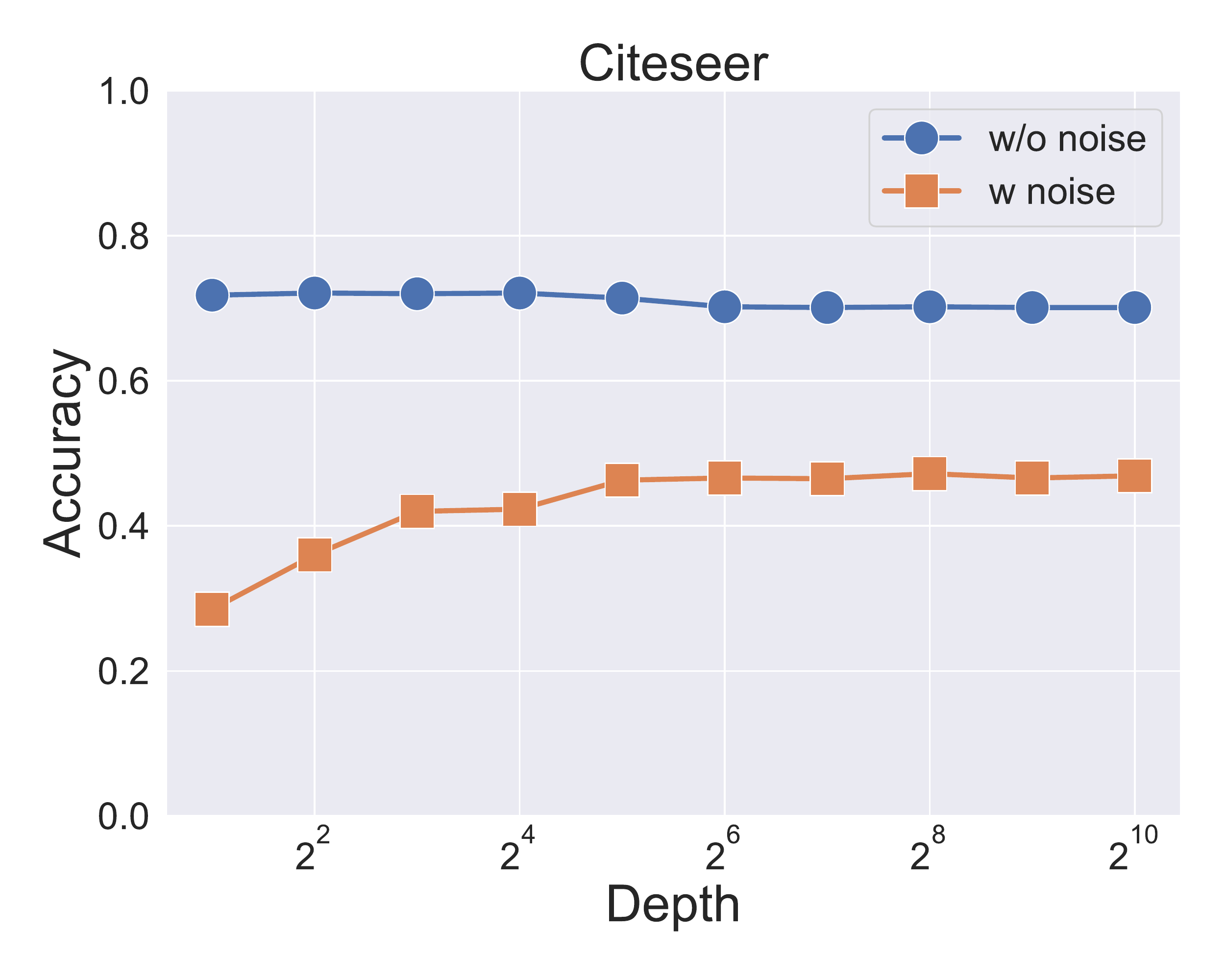}
        \includegraphics[width=0.325\textwidth]{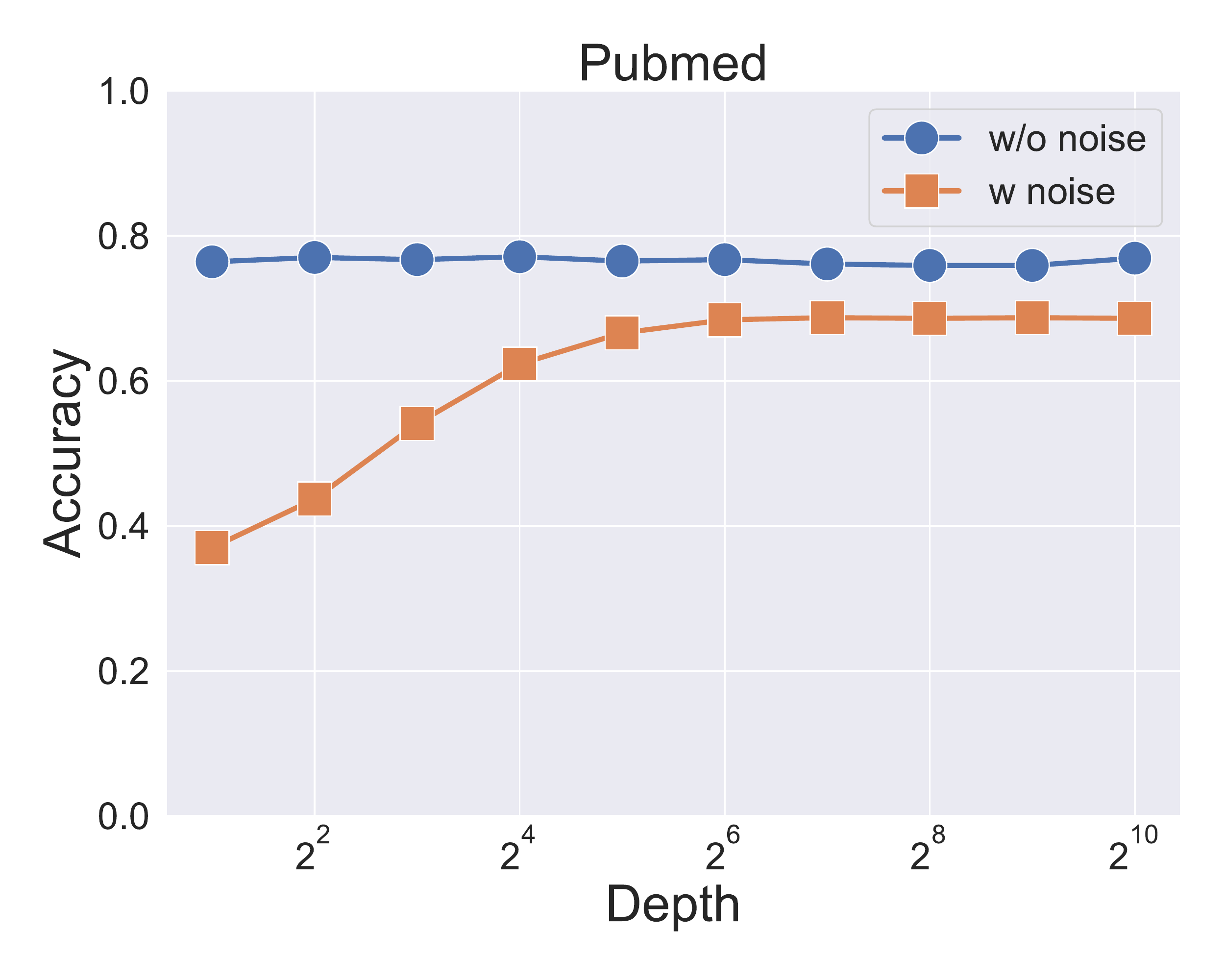}
        \end{center}
        \caption{Comparison of classification accuracy v.s. \name model depth on semi-supervised node classification tasks. The experiments are conducted on clean and noisy features.} 
        \label{fig:depth_ngc}
 \end{figure}
 
 \begin{figure}[t]
    \begin{center}
        \includegraphics[width=0.325\textwidth]{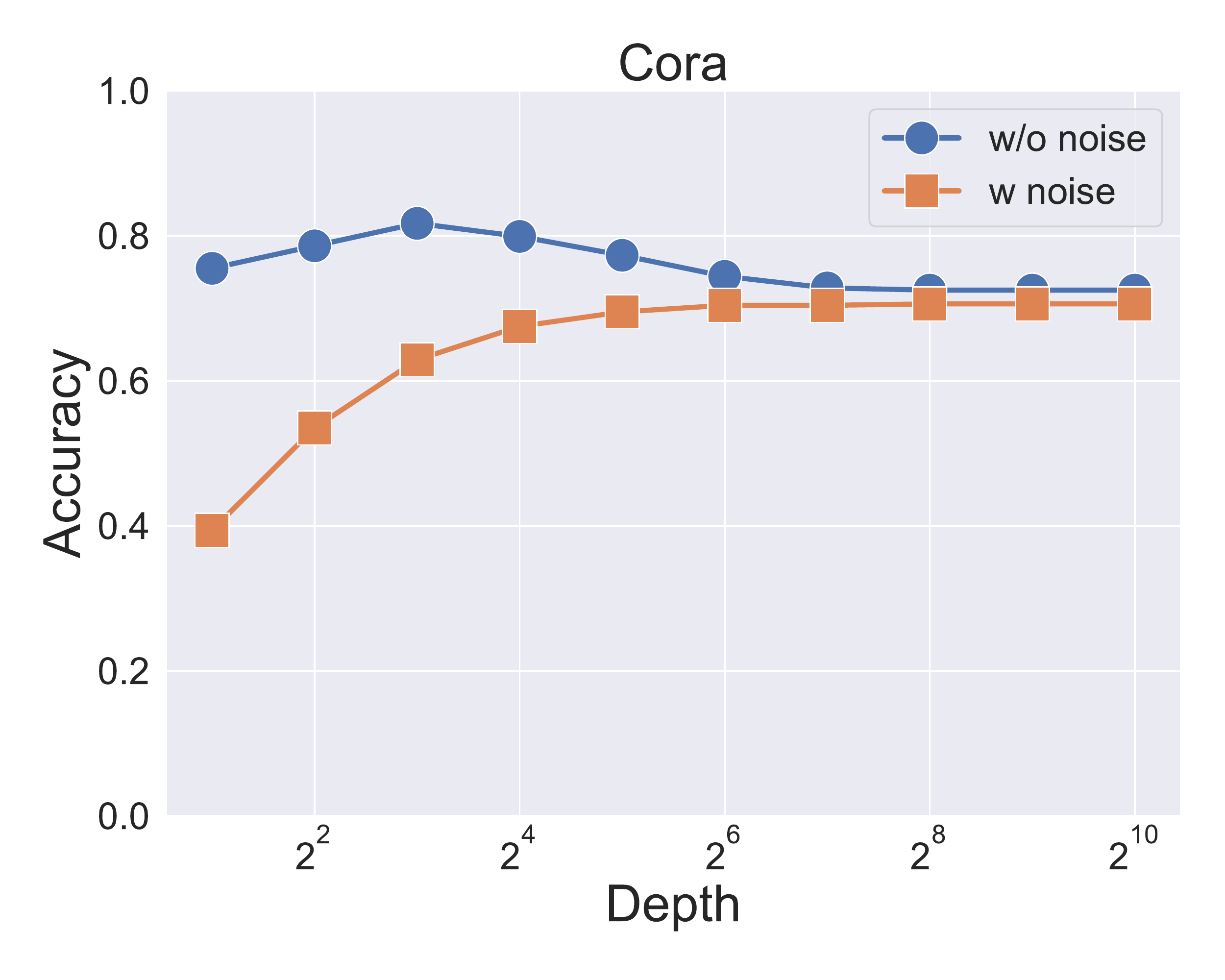}
        \includegraphics[width=0.325\textwidth]{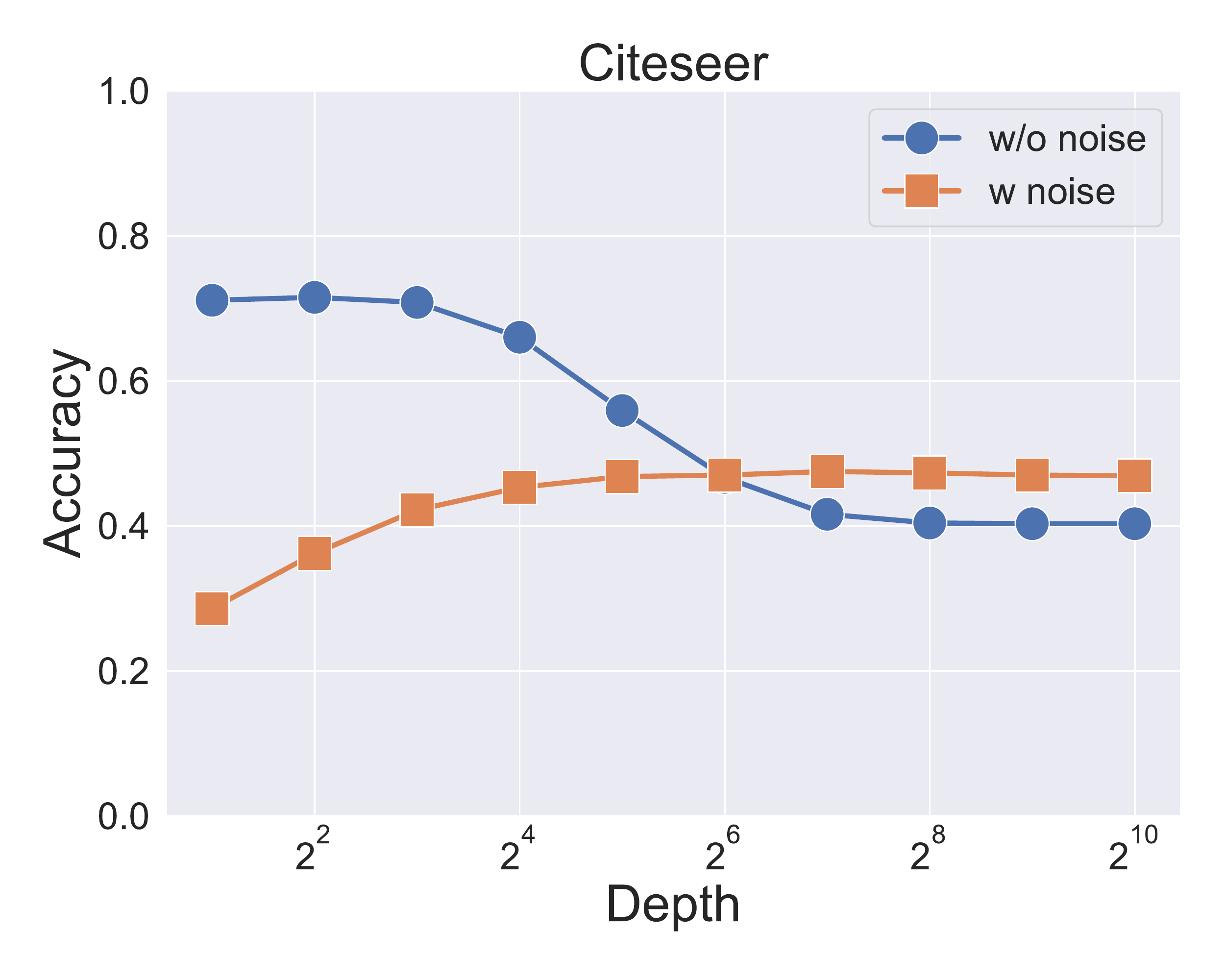}
        \includegraphics[width=0.325\textwidth]{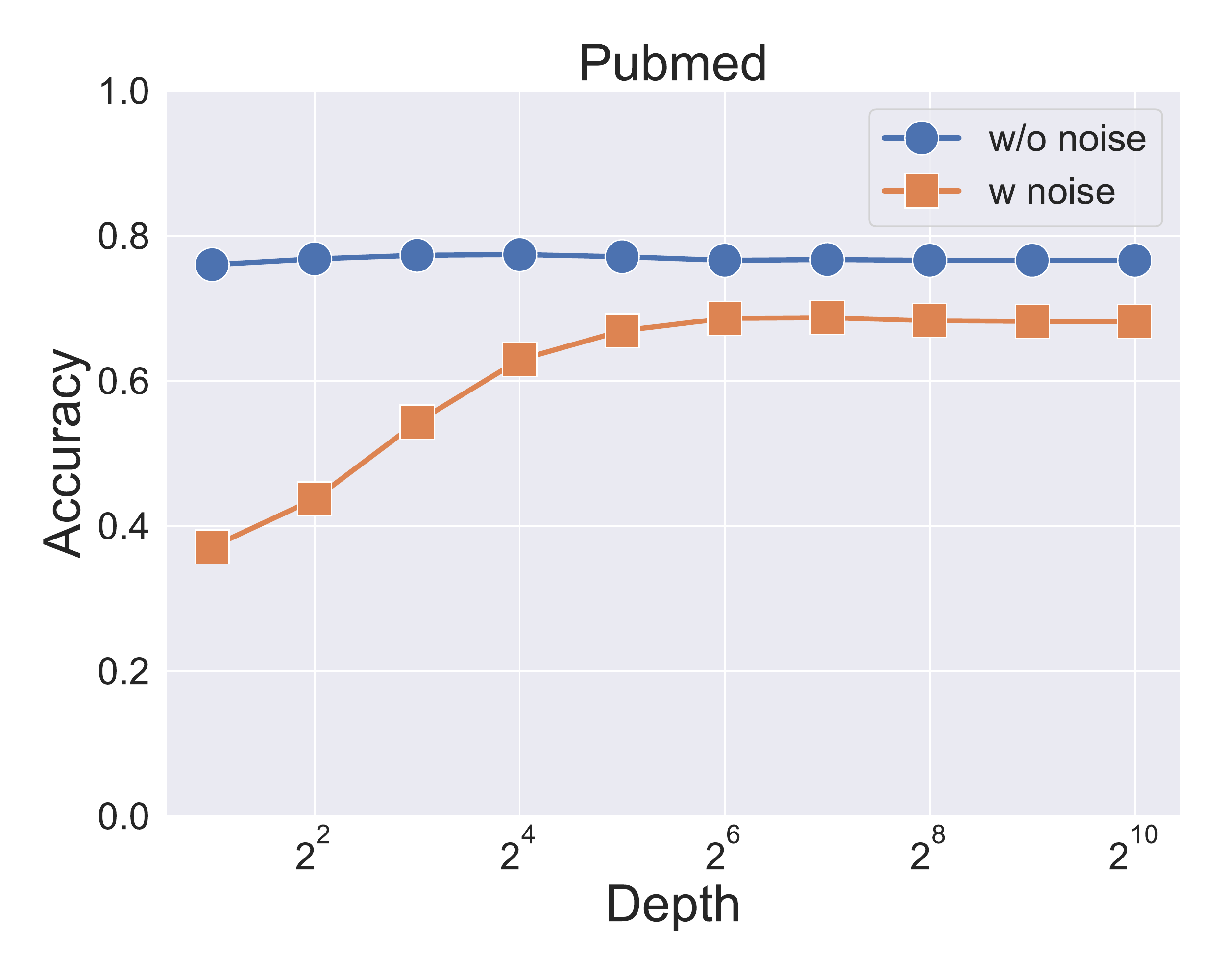}
        \end{center}
        \caption{Comparison of classification accuracy v.s. R\name model depth on semi-supervised node classification tasks. The experiments are conducted on clean and noisy features.} 
        \label{fig:depth_rngc}
 \end{figure}

\end{document}